\documentclass{article}

\usepackage{microtype}
\usepackage{graphicx}
\usepackage{subfigure}
\usepackage{booktabs} 
\usepackage{xspace}
\usepackage{multirow}
\usepackage{lscape}
\usepackage{hyperref}

\usepackage{ulem} 

\usepackage[accepted]{icml2024}

\usepackage{amsmath}
\usepackage{amssymb}
\usepackage{mathtools}
\usepackage{amsthm}
\usepackage{xcolor}
\usepackage[capitalize,noabbrev]{cleveref}

\usepackage{tikz}
\usepackage{lipsum}

\usetikzlibrary{shapes,decorations,arrows,calc,arrows.meta,fit,positioning}
\tikzset{
    -Latex,auto,node distance =1 cm and 1 cm,semithick,
    state/.style ={ellipse, draw, minimum width = 0.7 cm},
    point/.style = {circle, draw, inner sep=0.04cm,fill,node contents={}},
    bidirected/.style={Latex-Latex,dashed},
    el/.style = {inner sep=2pt, align=left, sloped}
}

\newcommand{\E}{\mathbb{E}}

\newcommand{\X}{\mathcal{X}}
\newcommand{\Y}{\mathcal{Y}}

\newcommand{\calF}{\mathcal{F}}

\newcommand{\argmin}{\text{argmin}}
\newcommand{\norm}[1]{\lVert #1 \rVert }
\newcommand{\vect}[1]{\mathbf{#1}}

\newcommand{\DR}{\textsc{DR}\xspace}
\newcommand{\DRSF}{\textsc{DR-SF}\xspace}
\newcommand{\MR}{\textsc{MR}\xspace}
\newcommand{\MRLC}{\textsc{MR - LC}\xspace}
\newcommand{\XGB}{\textsc{XGB}\xspace}
\newcommand{\MFSAN}{\textsc{MFSAN}\xspace}
\newcommand{\DADIL}{\textsc{DADIL}\xspace}

\theoremstyle{plain}
\newtheorem{theorem}{Theorem}[section]

\newtheorem{lemma}[theorem]{Lemma}

\theoremstyle{definition}

\newtheorem{assumption}[theorem]{Assumption}
\theoremstyle{remark}
\newtheorem{remark}[theorem]{Remark}

\usepackage[textsize=tiny]{todonotes}

\icmltitlerunning{Multiply Robust Estimation}

\begin{document}

\twocolumn[
\icmltitle{Multiply Robust Estimation for Local Distribution Shifts with Multiple Domains}



\icmlsetsymbol{equal}{*}

\begin{icmlauthorlist}
\icmlauthor{Steven Wilkins-Reeves}{uw}
\icmlauthor{Xu Chen}{comp}
\icmlauthor{Qi Ma}{comp}
\icmlauthor{Christine Agarwal}{comp}
\icmlauthor{Aude Hofleitner}{comp}
\end{icmlauthorlist}

\icmlaffiliation{uw}{Department of Statistics, University of Washington, Seattle, USA}
\icmlaffiliation{comp}{Central Applied Science, Meta, Menlo Park, USA}


\icmlcorrespondingauthor{Steven Wilkins-Reeves}{stevewr@uw.edu}
\icmlcorrespondingauthor{Xu Chen}{xuchen2@meta.com}

\icmlkeywords{Domain Adaptation, Multiply Robust, Covariate Shift, Label Shift}

\vskip 0.3in
]



\printAffiliationsAndNotice{}  

\begin{abstract}
Distribution shifts are ubiquitous in real-world machine learning applications, posing a challenge to the generalization of models trained on one data distribution to another. We focus on scenarios where data distributions vary across multiple segments of the entire population and only make local assumptions about the differences between training and test (deployment) distributions within each segment. We propose a two-stage multiply robust estimation method to improve model performance on each individual segment for tabular data analysis. The method involves fitting a linear combination of the based models, learned using clusters of training data from multiple segments, followed by a refinement step for each segment. Our method is designed to be implemented with commonly used off-the-shelf machine learning models. We establish theoretical guarantees on the generalization bound of the method on the test risk. With extensive experiments on synthetic and real datasets, we demonstrate that the proposed method substantially improves over existing alternatives in prediction accuracy and robustness on both regression and classification tasks. We also assess its effectiveness on a user city prediction dataset from Meta. 

\end{abstract}

\section{Introduction}
Machine learning models are often trained, evaluated and monitored across multiple segments in both scientific research and industrial applications \citep{perone2019unsupervised, wang2019bermuda, mcmahan2013ad, chen2023binary}. Due to the heterogeneous nature of data distributions, the optimal model and its performance may differ among these segments. For example, in the context of the technology industry, it is crucial to evaluate model performance across different countries, end devices and age groups, as user behaviors can significantly vary across these segments \citep{mcmahan2013ad}. In addition to the heterogeneity of data distributions across segments, it is also a common issue where the distribution of the data available for model training differs from that of the target population for model deployment \citep{storkey2009training}, primarily due to sampling bias \citep{bethlehem2010selection, gonzalez2014assessing}.

While building domain\footnote{We use `domain' and `segment' interchangeably hereafter.}-specific models for all segments and then adjusting for sampling bias may be a straightforward solution, it presents at least three challenges: (1) unreliable model training on segments with an insufficient number of labels, (2) excessive computation when training individual models with a large number of segments, and, (3) a potential loss of information shared across different segments. To address these challenges, we aim to propose a flexible modeling framework for general supervised learning scenarios on \emph{tabular data}, with the goal of improving model performance on each individual domain; however, we also discuss how our methods could be adopted to image and text data.
\begin{figure}[htb!]
    \begin{center} \centerline{\includegraphics[width=0.70\columnwidth]{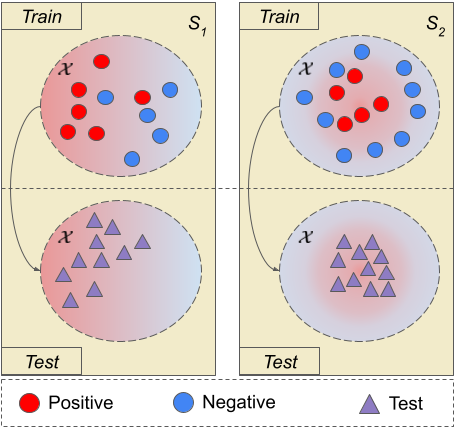}}
    \caption{Schematic representation of the problem setup with two segments $\{S_1, S_2\}$. We observe both features and responses for the training data and only features for the test domain. We don't make any global assumptions of distribution shifts but only shifts within a segment.}
    \label{fig:problem_setup}
    \end{center}
\vskip -0.3in
\end{figure}

Towards that end, we operate under the setting that the observed data in each domain adheres to the standard supervised learning setting. Specifically, we observe a training set and a test (or deployment) set in each domain. The training set consists of samples with a set of features and corresponding responses $(X,Y)$, while the test set only contains samples with features $X'$, and our objective is to predict their corresponding responses $Y'$. 

To make accurate predictions of $Y'$ under distribution shift, it is necessary to assume that the distribution of the training data $P$ and that of the test data $Q$ share a \emph{persistent property}, a component that is identical for $P$ and $Q$, to allow the model learned from the training data to be generalized to the test data. In classic domain adaptation literature, the two fundamental persistent property assumptions are \emph{covariate shift} \citep{shimodaira2000improving} and \emph{label shift} \citep{saerens2002adjusting}. In this work, we adopt a \emph{local} assumption that such persistent property only holds within each segment, without positing any restrictive \emph{global} assumptions on how data distributions are different across segments. We refer to this assumption as \emph{local distribution shift}. The problem setup we have described above is visualized in \cref{fig:problem_setup}.

Under the local distribution shift assumption, we propose a novel two-stage estimation method to enable both information sharing across segments and adaptive adjustments for each individual segment. In the first stage, we learn base models by utilizing the training data from multiple domains and then determine the optimal linear combination of these base models for each segment. In the second stage, we further refine the model from the first stage with a regularization to achieve an effective bias-variance trade-off. Inspired by the doubly robust estimation for covariate shift \citep{reddi2015doubly} and multiple robustness methods used in causal inference and missing data problems \citep{han2013estimation, chan2013simple}, we refer to the proposed method as `multiply robust’ estimation because it leverages all base models to create a domain-specific final model with enhanced predictive power and robustness.

Our main contributions are as follows.
\begin{itemize}
    \item We introduce a relaxed framework for domain adaptation, by assuming the persistent properties only hold locally within each segment of the training and test data. We do not make any assumptions about how the data distributions may vary across different segments.
    \item We propose the multiply robust estimation method for tabular data analysis that is widely applicable
    \begin{itemize}
        \item for both covariate shift and label shift scenarios.
        \item in general supervised learning scenarios, including regression, binary and multi-class classifications.
        \item with commonly used models such as linear models, tree-based models and more.
    \end{itemize}
    \item From a theoretical viewpoint, we establish guarantees on the generalization bound for the proposed method on the test risk. From an empirical perspective, we demonstrate that the proposed method substantially outperforms over existing alternatives in terms of both prediction accuracy and robustness, through comprehensive numerical experiments on both synthetic and real datasets. 
\end{itemize}

The rest of the paper is organized as follows. We review the related work and highlight the novelty of the proposed method in \cref{sec:related_work}. In \cref{sec:problem_setup}, we outline the problem formulation and assumptions along with notations to facilitate the introduction of the proposed method. We propose the multiply robust estimation method in \cref{sec:method} and provide pragmatic strategies for implementation. In \cref{sec:theory}, we justify the method with theoretical guarantees on generalization bound. We evaluate the empirical performance of the method with extensive experiments on both simulated and real datasets and compare it against several competing methods in \cref{sec:experiments}. We conclude and discuss several future research directions in \cref{sec:conclusion}.

\section{Related Work}
\label{sec:related_work}
We divide the related work broadly into two categories: single-source domain adaptation and multi-source domain adaptation.

\textbf{Single-source domain adaptation.} This category of research focuses on scenarios where all training data originates from a single-source domain. The two most common assumptions are covariate shift \citep{shimodaira2000improving} and label shift \citep{saerens2002adjusting}. A substantial body of literature proposes weighting based methods to adjust the training data to align it with the target population. Some prominent examples include Kullback-Leibler importance estimation \citep[KLIEP]{sugiyama2008direct}, kernel mean matching \citep[KMM]{gretton2009covariate}, penalized risk minimization \citep[PRM]{nguyen2010estimating} and discriminative learning based methods \cite{bickel2009discriminative} for covariate shift, and maximum likelihood label shift \citep[MLLS]{saerens2002adjusting} and black box shift estimation \citep[BBSE]{lipton2018detecting} for label shift. These weighting methods can be adopted to adjust local distribution shifts in each segment, as we shall elaborate in \cref{sec:importance_weights}. 

Despite the wide usage of weighting correction, it may introduce high variance due to the potential extreme weights \citep{reddi2015doubly}. For variance reduction, a doubly robust (\DR) method is proposed by incorporating a shrinkage towards the model trained without weights. Applying \DR in our setting requires combining training data from all segments. Therefore \DR is sub-optimal because it is unable to effectively adjust for any potential distribution shifts across segments, even when domain-specific features are incorporated into the modeling (\cref{sec:experiments}).

\textbf{Multi-source domain adaptation.} The work in this area assumes the presence of multiple source domains and aims to make predictions on a new target domain \citep{mansour2008domain}. A line of work in the avenue seeks to learn domain-invariant representations for all domains using deep learning \citep{zhao2018adversarial, xu2018deep}. However, such representation can be challenging to learn without strong assumptions on the shared properties across segments \citep{zhu2019aligning}. To address this difficulty, \citet{zhu2019aligning} proposes \MFSAN to extract both global and domain-specific features. Another research direction leverages optimal transport to learn a dictionary of data distributions and then approximate the target population using a mixture of the elements in the dictionary \citep[\DADIL]{montesuma2021wasserstein, montesuma2023multi}. We select two representatives, \MFSAN and \DADIL, from the two categories for comparison with the proposed method in numerical experiments in \cref{sec:experiments}.

\section{Problem Formulation}
\label{sec:problem_setup}
\textbf{Supervised learning.} We consider the standard framework for supervised learning. Let $X \in \X\subset\mathbb{R}^d$ denote the covariates (features), $Y \in \Y\subset\mathbb{R}$ denote the outcome of interest (label), and $f: \X \to \Y$ denote a prediction model, where $f \in \cal{F}$ for some function space $\cal{F}$. 

\textbf{Multi-Segment Unsupervised domain adaptation.} In the setting of domain adaptation, we consider $N$ segments that span the whole population. Let $S \in \mathcal{S}=\{1,2,...,N_S\}$ denote the segment index, $\{Y_i, X_i, S_i\}_{i = 1}^{n_{tr}}$ and $\{Y'_j, X'_j, S'_j\}_{j = 1}^{n_{te}}$ be the training and testing sets sampled independently from population probability distributions $P$ and $Q$ respectively, where $P \not = Q$. Our goal is to construct segment-specific predictors $f^{(s)}: \X \to \Y$ which perform well on each segment's test distribution $Q^{(s)} \equiv Q(y,x|S = s)$ according to the segment-conditional risk $R_Q[f|s] = \E_{(X,Y) \sim Q^{(s)}}[\ell(Y,f(X))]$, where $\ell(\cdot, \cdot): \Y\times \Y\to \mathbb{R}$ is a loss function. We consider unsupervised domain adaptation where the test labels $Y'_j$s are unobserved. We similarly define $P^{(s)} \equiv P(y,x|S = s)$.


\textbf{Local distribution shifts.} Classical assumptions include covariate shift \citep{shimodaira2000improving} where $P(y|x) = Q(y|x),~P(x) \not = Q(x)$ and label shift \citep{saerens2002adjusting} where $P(x|y) = Q(x|y), ~P(y) \not = Q(y)$. In our context, we adopt a relaxed version of these shifts  where either covariate shift or label shift holds within each segment throughout this work. These assumptions include:
\begin{itemize}
    \item \emph{Local covariate shift}: $P^{(s)}(y|x) = Q^{(s)}(y|x)$
    \item \emph{Local label shift}: $P^{(s)}(x|y) = Q^{(s)}(x|y)$.
\end{itemize}
We highlight the motivation for these relaxed assumptions with a toy example. Consider a scenario where the manifestation of symptoms $X$ (such as fever) is contingent upon the status of an influenza infection $Y$, and we wish to develop a classifier for this year's flu optimized for both children ($S=0$) and adults ($S=1$) that exhibit differences in symptoms of infection $P(x|y, s = 0) \not = P(x|y, s = 1)$. Here, $P$ represents the training distribution derived from historical data, while $Q$ signifies the test distribution for the current flu season. Across various flu seasons, the prior propensity of contracting the flu may vary, expressed as $P(y|s) \neq Q(y|s)$. However, the distribution of symptoms given a flu status remains consistent, i.e., denoted by $P(x|y, s) = Q(x|y, s)$, thereby illustrating a local label shift as opposed to a global shift.


\section{Multiply Robust Estimation}  
\label{sec:method}
In this section we introduce our proposed multiply robust estimation method. This approach allows for information sharing across segments and adaptive adjustments for each domain, resulting in enhanced model performance on each specific domain. An overview of the proposal is as follows. We first cluster the segments into several groups and train a base model on each cluster. This enables information sharing across segments with similar data distributions and particularly beneficial for segments with a limited number of samples to leverage strength from others. To adjust for local distribution shifts, we estimate the importance weights to match the training data distribution to test data for each segment. We then learn a domain-specific model by finding the best linear combination of the base models, which we refer as stage 1 estimators. With the stage 1 estimators as `priors', we further refine the estimators using weighted training samples. The pseudo-code of this procedure is presented in \cref{alg:MR_estimator}.

\begin{algorithm}[htb]
\begin{small}
\caption{Multiply Robust Estimator}
\label{alg:MR_estimator}
\begin{algorithmic}
  \STATE \textbf{INPUTS: }Training data: $\{Y_i, X_i, S_i\}_{i=1}^{n_{tr}}$. Testing data: $\{ X'_j, S'_j\}_{j=1}^{n_{te}}$. Weight estimator: $\mathcal{W}: (\{Y_i, X_i\}_{i=1}^{n_{tr}}, \{X_j'\}_{j = 1}^{n_{te}}) \to \hat w$. Base Model Training Algorithm ${\cal H} : \{Y_i, X_i\} \to h_m$. Segment Clustering Algorithm $\mathcal{C}_M: \{Y_i, X_i, S_i\} \to \{\mathcal{S}_m\}_{m = 1}^M$, ${\cal S}_m \subset {\cal S}$. Base training split $\varsigma \in [0,1]$\\
  \STATE \textbf{Step 1}: Randomly partition training indices into $N_{base}$, $N_{tune}$ such that $|N_{base}|/|N| = \varsigma$ \\
  \STATE \textbf{Step 2}: Cluster the segments $\mathcal{C}_M(\{Y_i, X_i, S_i\}_{i = 1}^{n_{tr}}) =  \{\mathcal{S}_m\}_{m = 1}^M $ into $M$ clusters and a final group partition which includes all clusters ${\cal S}_{M + 1} = {\cal S}$ where each $\mathcal{S}_m \subseteq \mathcal{S}$ are subsets of the set of segments.\\
  \STATE \textbf{Step 3}: For each $m$, train a base model using the segments $h_m = {\cal H}\left( \{Y_i, X_i\}_{i \in N_{base} : S_i \in \mathcal{S}_m}\right)$ \\
  \FOR{$s \in \mathcal{S}$}
        \STATE \textbf{Step 4a}: Estimate the weights for segment $s$, $\widehat w_s = \mathcal{W}(\{Y_i, X_i\}_{i = 1: S_i = s}^{n_{tr}}, \{X_j'\}_{j = 1: S'_j = s}^{n_{te}})$
        \STATE \textbf{Step 4b}: (Stage 1 estimator). Using the tuning segment data $\{Y_i, X_i\}_{i \in N_{tune}: S_i = s}$, learn a linear combination of the base classifiers $\widehat f^{1, (s)} = \sum_{m = 1}^M \widehat \beta^{(s)}_m h_m$. 
        \STATE \textbf{Step 4c}: (Stage 2 estimator). Refine a second model $\widehat f^{(s)}_{MR}$  using the training data within the segment $\{Y_i, X_i\}_{i \in N_{tune}: S_i = s}$ weighting each sample using the weights learned in step 4a, and penalizing the difference of $\widehat f^{(s)}_{MR}$ to $\widehat f^{1, (s)}$ 
  \ENDFOR
  \STATE \textbf{OUTPUT:} Collection of refined estimators $\{\widehat f^{(s)}_{MR}\}_{s \in {\cal S}}$
\end{algorithmic}
\end{small}
\end{algorithm}


We draw inspirations from doubly robust method \cite{reddi2015doubly} and multiply robust methods used in causal inference and missing data problems \cite{han2013estimation, chan2013simple}. We refer our proposal as `multiply robust' because the predictive power of any base model can be leveraged in stage 1 estimators. If the linear combination is sub-optimal, further refinement is achieved with stage 2 estimators.

The key innovation of the proposed method is the two-stage estimation, which we delve into the details in \cref{sec: Two stage estimation}. In \cref{sec:training_base_models}, we provide pragmatic implementation strategies for the other components in \cref{alg:MR_estimator} including segment clustering and weighting. We discuss the computational complexity of the proposed method in \cref{sec:computation}.

\begin{remark}
In applications that involve image and text data, it is common to exploit large-scale pre-trained models such as VGG-16 \cite{simonyan2014very}. A straightforward and widely adopted approach to applying our method to text and image data is to utilize the embeddings of a large pre-trained model, treating them as the feature vector. Subsequently, our method can be applied, as outlined in \cref{alg:MR_estimator}, employing a regularized, high-dimensional linear model for both the base-model and the stage 2 estimator classes. While addressing label shift issues typically involves straightforward application of BBSE, caution should be exercised when dealing with high-dimensional embeddings, necessitating careful selection of an appropriate covariate shift weight estimation method. 
\end{remark}

\subsection{Two Stage Domain-adaptive Estimation} \label{sec: Two stage estimation}

We now elaborate the two-stage estimation method. Let $\{h_m\}_{m = 1}^M, h_m : {\cal X} \to { \cal Y}$ denote a set of base models and $\vect{h}$ the vector collection of them. Let $R[f|\mathbf{Y}, \mathbf{X}, ,s; \alpha]$ denote the $\alpha$-weighted empirical training risk of segment $s$, 
$$ R[f|\mathbf{Y}, \mathbf{X},s; \alpha] = \frac{1}{n_s}\sum_{i \in N_s}\alpha(Y_i, X_i) \ell(Y_i, f(X_i))$$
where $N_s$ denotes the indices of the data corresponding to segment $s$: $N_s \subset \{1,2,\dots n\}, n_s = |N_s|$, and $\mathbf{Y}, \mathbf{X}$ denote the entire training data. We abbreviate the unweighted empirical risk $R[f|\mathbf{Y}, \mathbf{X},s; \vect{1}] = R[f|\mathbf{Y}, \mathbf{X},s]$. We let $\mathcal{B} \subseteq \mathbb{R}^M$ denote a convex subset for the parameters. Lastly, let $\Omega[f,g]$ denote a penalty function for the deviation of $f$ and $g$ (for example $\norm{f - g}_{\calF}$ for a function space ${\cal F}$, or $\norm{\beta_f - \beta_g}_2$ for the parameters of a linear model), where $\Omega[f] = \Omega[f,0]$. Given a set of estimated weights on the test set $\widehat w_s$, and a regularization parameter $\lambda'_s$\footnote{This can be any type of regularization, such as restricting the number of trees in an ensemble, or penalizing the coefficients in a linear model.}. Steps 4b and 4c are presented as follows 
\begin{align*}
    \widehat \beta^{(s)} &= \argmin_{\beta \in {\cal B}} R[\beta^\intercal \vect{h} | \mathbf{Y}, \mathbf{X},s] \\
    &\text{(Stage 1 Estimator)} \\
    \widehat f^{(s)}_{MR} &= \argmin_{f \in {\cal F}} R[ f| \mathbf{Y}, \mathbf{X},s; \widehat w_s] + \lambda_s' \Omega[\widehat \beta^{(s) \intercal} \vect{h} , f]\\
    &\text{(Stage 2 Estimator)} 
\end{align*}
\subsubsection{Stage 1 Estimators} \label{sec: stage 1 estimators}
When ${\cal B} = \mathbb{R}^M$, the stage 1 estimator is a simple linear combination of base models. For example, if the loss function is the mean square loss this can be computed using linear regression or logistic regression when the $\ell$ is the cross entropy loss. In practice, the performance of the model at this stage is influenced by the quality of the base models. Therefore, as with all machine learning problems, it is crucial to select a method that is suitable for the specific prediction task at hand, including kernel methods, ensemble methods, and others.

To avoid overfitting in the first stage, one may consider restrictions on ${ \cal B}$. Our default restriction is to consider the unit ball of coefficients $\beta \in \mathcal{B}_1 = \{\beta : \norm{\beta}_2 \leq 1 \}$, which includes convex combinations of the original model set. The unit ball restriction can be simply implemented through a successive application of ridge regression. for a specific regularization parameter $\lambda$
\begin{align*}
    \widehat \beta_{\lambda} &= \argmin_{ \beta \in \mathbb{R}^M} \sum_{i \in N_s}\ell(y_i, \beta^\intercal \vect{h}(X_i)) + \frac{\lambda}{2}\norm{\beta}^2 \\
    =  &\argmin_{\beta: \norm{\beta}_2 \leq \nu}  \sum_{i \in N_s}\ell(Y_i, \beta^\intercal \vect{h}(X_i)) 
\end{align*}
there is a corresponding $\nu$ corresponding to the constrained optimization problem. We can find the corresponding $\lambda$ via a binary search to find the smallest value of $\lambda \leq \lambda_{max}$ such that $\norm{\widehat \beta_{\lambda}}_2 \leq 1$, for some preset $\lambda_{max}$.



\subsubsection{Stage 2 Estimators}
The stage 2 estimator requires penalization towards a base model. In the case of the mean square loss and a form of the penalty $\Omega[f,g] = \Omega[f - g]$, then this can be computed by fitting a regularized model fit to the residuals of the stage 1 model, $\delta Y_i = Y_i - \widehat \beta^{(s) \intercal} \vect{h}(X_i) $. For other loss functions, such as those involved in classification, one can include a base margin $\widehat \delta_i = \widehat \beta^{(s) \intercal} \vect{h}(X_i) $ then one can learn a model $\tilde g$ such that $f = \tilde g + \widehat \delta_i$ by penalizing the complexity of $\tilde g$, through $\Omega[\tilde g]$. This functionality is available in popular packages such as XGBoost \citep[\XGB]{chen2016xgboost} and can be similarly implemented in many common function classes such as SVMs, tree-based methods and neural networks. In the case of XGBoost $\Omega[\tilde g]$ involves penalties on the complexity of the second stage model, i.e. the number of trees, tree depth etc. We illustrate a cross validation procedure for our method in \cref{sec: cross_validation}.

\subsection{Training Base Models and Learning Weights}
\label{sec:training_base_models}
To complete the understanding of \cref{alg:MR_estimator}, we discuss pragmatic strategies for the remainder of the steps on obtaining base models and importance weights. For the former, since we use off-the-shelf methods for training the base models, we focus on segment clustering below.

\subsubsection{Segment Clustering for Base Model Training}\label{sec: clustering for training base models}
The trade-off in training base models is that, it is desirable to train models on more segments to reduce variance, while the variability across base models is also crucial to provide a richer function set for the stage 1 estimator. We balance the trade-off by first clustering the segments into groups $\mathcal{S}_m \subset \mathcal{S}$, based on the similarity of covariates and responses, and then training one base model using all the training data within each segment cluster $\{Y_i, X_i\}_{i: S_i \in \mathcal{S}_m }$. 

One can cluster the segments with domain knowledge, such as a continent when the domains are countries. When external knowledge is not available to define clusters of segments, we cluster the segments into groups where the joint training distribution $P^{(s)}(y,x)$ are similar for $s \in \mathcal{S}_m$. A common choice to measure the similarity between two distributions $P$ and $Q$ is Maximum Mean Discrepancy (MMD) \cite{smola2006maximum}. With empirical measures $\widehat P, \widehat Q$ and corresponding samples $\{u_i\}_{i = 1}^{m_1} $ and $\{v_i\}_{i = 1}^{m_2}$, MMD can be computed explicitly in \eqref{eq: MMD_empirical} where $k(\cdot , \cdot)$ is a positive definite, symmetric kernel function. 
\begin{align}
    D_{MMD}(\widehat P, \widehat Q) &= \frac{1}{m_1(m_1 - 1)} \sum_{i = 1}^{m_1}\sum_{j \not = i}^{m_1} k(u_i,u_j) \label{eq: MMD_empirical}\\
    + \frac{1}{m_2(m_2 - 1)} &\sum_{i = 1}^{m_2}\sum_{j \not = i}^{m_2} k(v_i,v_j) - \frac{2}{m_1m_2} \sum_{i = 1}^{m_1}\sum_{j = 1}^{m_2} k(u_i, v_j) \nonumber
\end{align}
We use \cref{eq: MMD_empirical} to compute a distance matrix between segments and perform hierarchical clustering as described in \cref{alg:segment_clustering}. The choice of kernel may be determined by the datatype. For example, we choose the delta kernel $k_y(y,y') = 1 - I(y = y')$ for categorical dimensions, where $I$ is the indicator function, and choose the Gaussian kernel for continuous dimensions. The joint kernel $k_{y,x}$ can be computed as the product of the $k_x$ and $k_y$ kernels, $k_{y,x}((y,x), (y',x')) = k_y(y,y')k_x(x,x')$. In our experiments, we generally adhere to a rule of thumb where we select the largest $M$ in \cref{alg:segment_clustering} that avoids generating singleton clusters, thereby ensuring a reasonably large function class in the base models. However, we note that generating an excessive number of distinct clusters could potentially restrict information sharing across segments, thereby diminishing the overall performance. One could also treat $M$ as a tuning parameter and select it through cross validation.

\begin{algorithm}[htb]
\begin{small}
\caption{Clustering Segments for Base Model Learning}
  \label{alg:segment_clustering}
\begin{algorithmic}[1]
  \STATE {\bfseries Input:} A kernel function: $k(\cdot, \cdot)$ and a training dataset $\{(Y_i,X_i,S_i)\}_{i = 1}^{n_{tr}}$, number of desired clusters $M$.
  \FOR{ each segment pair $(s, s') \in \mathcal{S}^2$ }
  \STATE Compute $\widehat D_{s s'} = D_{MMD}(\widehat P^{(s)}( y, x), \widehat P^{(s)}(y, x))$ where $\widehat P^{(s)}(y, x)$ corresponds to the empirical distribution of training data $(Y,X)$ in segment $s$.
  \ENDFOR
  \STATE Apply hierarchical agglomerative clustering to $\widehat D$ until $M$ clusters are obtained.
\end{algorithmic}
\end{small}
\end{algorithm}

\subsubsection{Importance Weight Estimation} 
\label{sec:importance_weights}
Lastly, we discuss options for addressing local distribution shifts. In classical covariate or label shift, one can weight the observations in the training distribution to match the test distribution. These sample weights can be used on the training distribution to debias the empirical training risk when the goal is minimizing the test risk. These weights, known as the importance weights, under covariate shift are defined by the ratio of densities (when $X$ is continuous), $w(x) = \frac{dQ(x)}{dP(x)}$ (also known as the Radon-Nikodym derivative), and under label shift are defined by the ratio of class frequencies $w(y) = \frac{Q(y)}{P(y)}$.

Under our local covariate and label shift assumptions, we can define the analogous importance weights conditional on a segment $s$. 
\begin{align*}
    w_s(x) &= dQ^{(s)}(x)/dP^{(s)}(x) \text{ local covariate shift} \\
    w_s(y) &= Q^{(s)}(y)/P^{(s)}(y) \text{ local label shift} 
\end{align*} 
We can then estimate the weights using the methods discussed in \cref{sec:related_work} in each segment, such as PRM or KMM for covariate shift or BBSE for label shift. 

\subsection{Computational Complexity}
\label{sec:computation}
The major computation burden in the proposed multiply robust method occurs at the segment clustering and 2-stage estimation. In \cref{alg:segment_clustering} for segment clustering, the distance matrix computation requires $\mathcal{O}(n_{tr}^2)$ operations, which is typically the most computationally expensive part of the procedure. The agglomeration clustering step has  complexity $\mathcal{O}(M N_S^2)$ which tends to be much smaller as $N_S \ll n_{tr}$. 


Suppose the complexities of training stage 1 and stage 2 estimators are $c_1(n_s)$ and $c_2(n_s)$ for segment $s$ in step 4 of \cref{alg:MR_estimator}.
The total algorithmic complexity is $\sum_{s=1}^{N_S} c_1(n_s) + c_2(n_s)$. For example, if we adopt linear regression in stage 1 and \XGB in stage 2, then $c_1(n_s) = \mathcal{O}(n_s M^2 + M^3)$ and $c_2(n_s) = \mathcal{O}(TD_Tdn_s + pn_s\log(n_s))$ resulting in the total complexity scaling log-linearly with $n_{tr}$. Here $T$ is the total number of trees and $D_T$ is the maximum tree depth and $d$ is the dimension of $X$. The base model training can be analyzed similarly as the stage 2 estimator, but for training $M$ models. It is worth noting that each of the stage 1 and 2 estimators can be computed in parallel, also allowing for further algorithmic improvements in practice.

\section{Theoretical Analysis}
\label{sec:theory}
We now present a theoretical justification for the proposed multiply robust estimation method. As our goal is to construct a model that is adaptive to each segment with high predictive accuracy, we provide a performance guarantee by establishing a generalization bound on the test risk. Our proof leverages results from \citet{ostrovskii2021finite} to establish the convergence rate of the stage 1 model. We integrate a modified version of the risk minimization technique presented in \citet{reddi2015doubly}. Notably, our methods are adaptable to any weight estimator, allowing us to separate the bound based on the quantities associated with the weight estimator $\widehat w_s$ and the function estimation of $\widehat f^{(s)}_{MR}$. The assumptions herein, though restrictive, are used to provide intuition on the performance of the algorithm and are not necessary for the algorithm to be effective in practice. 


Recall the training risk for segment $s$, $R_P[f|s] = \E_{(Y,X) \sim P^{(s)}}[\ell(Y,f(X))]$ and denote the corresponding test risk $R_Q[f|s] = \E_{(Y,X) \sim Q^{(s)}}[\ell(Y,f(X))]$, and $\ell(y,\varrho)$ is a loss function which is convex in $\varrho$ for each $y$, such as the mean square loss for regression or cross entropy loss for classification. Our goal is to minimize the test risk $R_Q[f|s]$ for any given segment $s$. We focus on one segment to establish the theory, which can then be applied to all of the segments. 

Suppose a set of base models $\vect{h}=\{h_m\}_{m=1}^M$ and the candidate model $f$ are contained in a Reproducing Kernel Hilbert Space (RKHS) $\calF$ equipped with the norm $\norm{\cdot}_{{\cal F}}$. Let $\beta^{*(s)} = \argmin _{\beta \in {\cal B }}R_P[\beta^\intercal \vect{h}|s]$ be the optimal linear combination and denote the following quantities derived from the base models
\begin{align*}
    Z_{mi} &= h_m(X_i),  \quad H = \E[\nabla^2_{\beta} \ell(Y,Z^\intercal\beta^{*(s)})|s] \\
     G &= \E[\nabla_{\beta} \ell(Y,Z^\intercal\beta^{*(s)}) \nabla_{\beta} \ell(Y,Z^\intercal\beta^{*(s)})^\intercal |s] \\
    M_{eff} &= \text{Tr}\left( H^{-1} G\right), \quad M_F = \max_{m} \norm{h_m}_{\calF} \\
    \Sigma &= \E[ZZ^\intercal|s], \quad \text{Where } Y_i, X_i \sim_{iid} P^{(s)}(y,x).
\end{align*}


Let $\ell'(y,\varrho) = \frac{d}{d\varrho} \ell(y,\varrho)$ and similarly for higher order derivatives. Assume that each $Q^{(s)}$ is absolutely continuous with respect to $P^{(s)}$ such that there exists a set of bounded importance weights $w_s(x,y) = \frac{dQ^{(s)}(x,y)}{dP^{(s)}(x,y)}$. 
\begin{assumption}[Regularity Conditions] \label{assumption: regularity conditions}
    We assume that the following regularity conditions hold: 
\begin{align}
    |\ell(y,\varrho)| &\leq L, \quad |\ell(y,\varrho) - \ell(y,\varrho')| \leq L |\varrho - \varrho'| \label{assumption: bounded Lipschitz}\\
    |\ell'''(y,\varrho)| &\leq |\ell''(y,\varrho)|, \quad \text{ for all } \varrho \in \mathbb{R} \label{assumption: self concordance}\\
    Z_{mi}, \ell'(Y, Z^\intercal&\beta^{*(s)})Z,  \ell''(Y,Z^\intercal\beta^{*(s)} )^{1/2}Z \sim \text{subG} \label{assumption: SubGaussian}\\ 
    \Sigma &\preceq \rho  H \text{ for some } \rho > 0  \label{assumption: Covariance Condition}\\
    \tilde{\lambda}_1 &= \lambda_{min}( H) > 0  \label{assumption: Alignment of Base Models} \\
    w_s(y,x) &\leq \eta \text{ for all } x \in {\cal X},y \in {\cal Y}  \label{assumption: bounded weights} 
\end{align}
where $ \xi \sim \text{subG}$ means that the $\xi$ is sub-Gaussian (i.e., $\mathbb{E}[\exp(s(\xi - \mathbb{E}[\xi]))] \leq \exp(\sigma^2s^2/2)$ for some $\sigma>0$ and for all $s \in \mathbb{R}$), and $\preceq$ denotes the semidefinite ordering. 
\end{assumption}

\begin{remark}
In \cref{assumption: regularity conditions}, \cref{assumption: bounded Lipschitz} is a simple boundedness and Lipschitz continuous property for the loss function. \cref{assumption: self concordance} is the 2-self concordance property and \cref{assumption: SubGaussian,assumption: Covariance Condition} are standard regularity assumptions for the finite sample convergence of parametric models as illustrated in \citet{ostrovskii2021finite} and satisfied for many convex loss functions such as those arising from exponential families. \cref{assumption: Alignment of Base Models} ensures that the optimal $\beta^{*(s)}$ is identifiable, as no model will be a linear combination of the others, and \cref{assumption: bounded weights} ensures that the true importance weights are bounded.
\end{remark}

For each segment, let $f^{*(s)}_{q, \nu_{q,s}}$ denote the minimizer of the constrained test risk 
\begin{align}
    f^{*(s)}_{q, \nu_{q,s}} &= \argmin_{f \in \calF: \norm{f}_{\calF} \leq \nu_{q,s}}R_Q[f|s] \label{eq:population_shift_solution}
\end{align}
where the target function class is bounded since an unregularized solution may not have a bounded norm. 
Let $\widehat w_s$ denote an estimate of the segment-specific weights $w_s$.

\begin{theorem}[Generalization Bound For a Multiply Robust Estimator]\label{thm:multiple_robust_rate}
Suppose that \cref{assumption: regularity conditions} holds. Then denote the 2-stage estimator illustrated in \cref{alg:MR_estimator},
\begin{align*}
    \widehat \beta^{(s)} &= \argmin_{\beta \in \mathcal{B}} R[\vect{h}^\intercal\beta|Y, X, s] \\
    \widehat f_{MR}^{(s)} &:= \argmin_{f \in \mathcal{F}} R[f|Y, X,s;\widehat w_s] + \lambda'_s \norm{f - \widehat \beta^{(s) \intercal} \vect{h}}^2_{\calF} \\
    &= \argmin_{f \in \mathcal{F}: \norm{f - \widehat \beta^{(s) \intercal} \vect{h}} \leq \nu'} R[f|Y, X,s,\widehat w_s] 
\end{align*}
where $\nu' = \norm{\beta^{*(s) \intercal}\vect{h} - f^{*(s)}_{\nu_{q,s}}}_{\calF} + C\frac{M_F}{\tilde{\lambda}_1}\sqrt{\frac{M \log(e/\delta)}{n_s}}$ and $C$ is a universal constant. 
For $n \geq  M_{eff} M \log(e M / \delta)$   then with probability at least $1 - 2\delta$
\begin{equation}
    R_Q[\widehat f_{MR}^{(s)}|s]  \leq R_Q[f^{*(s)}_{\nu_{q,s}}|s] + \Delta_{MR, w_s, s} + \Delta_{MR, f, s}
\end{equation}
where $ \Delta_{MR, w_s, s}$ is the bound related to the learning of the correct weights, and $\Delta_{MR, f, s}$ is the functional complexity part of the bound. 
\begin{align*}
    \Delta_{MR, w_s} &:= 2L\frac{\norm{\widehat w_s - w_s}_1}{n_s} \\
    \Delta_{MR, f, s} &:= 2\eta L \bigg( \frac{2\nu'\sqrt{\text{tr}(\vect{K}_s)}}{n_s} + 3\sqrt{\frac{\log(2/\delta)}{2n_s}} + \\
    &\frac{\norm{\widehat \beta^{(s) \intercal} \vect{h}(\vect{X}_s)}_2}{n_s}\bigg) 
\end{align*}
where $\norm{\widehat w_s - w_s}_1 = \sum_{i \in N_s} | \widehat w_s(X_i,Y_i) - w_s(X_i,Y_i)|$ and $\vect{K}_s$ is the Gram matrix corresponding to the kernel function $k(\cdot, \cdot)$ associated with the space ${\cal F}$, $\vect{K}_{s, ij} = k(X_i,X_j): i,j \in N_s$, and $\norm{\widehat \beta^{(s) \intercal} \vect{h}(\vect{X}_s)}_2 = \sum_{i \in N_s}(\widehat \beta^{(s) \intercal} \vect{h}(X_{i}))^2$
\end{theorem}
\begin{remark}[Error decomposition]
We split the components $\Delta_{MR, w_s}$ and $\Delta_{MR, f, s}$ so that the estimation rate of our method is separated from that of learning the importance weights. We discuss the estimation rates for different weighting methods in \cref{appendix: importance weights estimation bounds}. In particular, PRM and BBSE methods can guarantee estimation of the correct weights, at rates of $\Delta_{MR, w_s, s} = \mathcal{O}(n_s^{-1/4})$ and $\mathcal{O}(n_s^{-1/2})$ respectively. 
\end{remark}

\begin{remark}[Comparison with global estimators]
Unless the minimizers $f^{*(s)}_{\nu_{q,s}}$ are identical across all segments $s$, any single competing estimator $\tilde f \in \mathcal{F}$ with $\norm{\tilde f}_{\mathcal{F}} \leq \nu_{q,s}$ will be asymptotically sub-optimal across every group. This includes global estimators that are trained on all segments, such as the standard doubly robust (\DR) estimator we compare in \cref{sec:experiments} \cite{reddi2015doubly}. 
\end{remark}

\begin{remark}[Comparison with doubly robust estimator]
An alternative competitor one could suggest is to apply the \DR estimator separately over every segment. Instead of employing a linear combination for the first stage, the \DR learns an estimate $\widehat f^{(s)}_{p, \nu_{p,s}} \in {\cal F}$ trained on the unweighted empirical risk. Our comparison focuses on the functional complexity component of our bounds. They include a term analogous to our $\nu'$ in their bound, which we denote as $\nu_{DR}'$. However, they demonstrate $\nu_{DR}' =\norm{f^{*(s)}_{\nu_{p,s}} - f^{*(s)}_{\nu_{q,s}}}_{{\cal F}} + \mathcal{O}(n_s^{-1/4})$, where $f^{*(s)}_{\nu_{p,s}}$ represents the population optimizer of the unweighted problem. In contrast, our first stage estimator is expressed as $\nu' = \norm{\beta^{*(s) \intercal}\vect{h} - f^{*(s)}_{\nu_{q,s}}}_{\calF} + \mathcal{O}(n_s^{-1/2})$, suggesting an improved rate for segments with small samples. Moreover if $\beta^{*(s) \intercal}\vect{h}$ proves to be a better approximation through similar segments to $f^{*(s)}_{\nu_{q,s}}$ then we can improve on the constant factor as well. 
\end{remark}
\begin{remark}[Curse of dimensionality]
In this analysis, it's important to acknowledge that our risk minimization is centered around comparing the estimator to the population minimizer within the function class of the Reproducing Kernel Hilbert Space (RKHS), thereby circumventing the curse of dimensionality. However, if one aims to minimize $R_Q[\cdot|s]$ across a broader function class we would expect to observe a dependence on the dimension of the features. Additionally, under covariate shift, the curse of dimensionality could potentially affect the estimation of the weights $\Delta_{MR, w_s}$. Nonetheless, our theoretical findings offer valuable insights into our method's behavior. 
\end{remark}

\section{Experiments} 
\label{sec:experiments}
We evaluate the methods on both simulated and real data. We use \textbf{mean squared error} (MSE) for regression tasks and \textbf{cross entropy} (CE) for classification tasks for evaluation. Standard errors are reported in parentheses in the following tables. More implementation details such as weights estimation and model hyperparameters are included in \cref{appendix: real data experiment details}. The code is available at \href{https://github.com/facebookresearch/multiply_robust}{github.com/facebookresearch/multiply\_robust}.




We compare the proposed multiply robust method (\MR) against several existing alternatives. We include a doubly robust\footnote{Though doubly robust methods were introduced to adjust for covariate shift, they can also handle label shift.} model (\DR)  and a version of \DR that simply includes the segments as hot features, which we refer to as \DRSF. We also compare two multi-source domain adaptation methods \MFSAN \cite{zhu2019aligning} and \DADIL \cite{montesuma2023multi}. Since tree ensemble methods have demonstrated superior performance on tabular data learning \citep{prokhorenkova2018catboost, shwartz2022tabular}, we also include \XGB as a baseline and use \XGB to construct base models for both \MR and \DR.



In summary, we find that the proposed \MR consistently outperforms the competitors. In particular, \MR achieves the lowest test losses on the majority of segments (see results in \cref{fig: relative_CE} and \cref{appendix: empirical_results}) on all datasets. Since \XGB and \DR do not account for the distributional heterogeneity across segments, their performances are typically sub-optimal. The improvement of \DRSF over \DR in most cases suggests that including one-hot segment feature helps the model adapt to different segments, though with limited capability. The performance of \MFSAN and \DADIL are sensitive to hyperparameter settings. Even after fine-tuning (see \cref{appendix: tuning_parameters} for more details), they still cannot achieve comparable performance to \MR and \DRSF. The issue with \MFSAN lies in its attempt to align the predictions made by all the classifiers learned from different sources (Equation (6) in their paper). This alignment penalty in \MFSAN could potentially hinder its performance, as it contradicts the goal of leveraging the heterogeneity of data from various sources. On the other hand, the dictionary learning approach, which is the core of \DADIL, assumes that the data distribution in each target domain can be well represented by a mixture of learned distributions. However, this assumption is restrictive and may not hold in real-world applications. Even if this assumption holds, effectively learning the dictionary of moderate to high dimensional densities becomes challenging due to the curse of dimensionality. In addition, it is worth noting that \MFSAN and \DADIL require significantly more computation time compared to other competing methods even with the use of GPU\footnote{Details on computing environment are in \cref{appendix: tuning_parameters}} (\cref{tab: locations-test-metrics}).




\subsection{Simulations}
We first consider a simple covariate shift regression simulation. We simulate a response $y =\alpha_{0s} + \alpha_{1s}^\intercal x + \gamma_s x_{2}(1 + \sin(3x_1)) + \epsilon_i$ where $\epsilon_i \sim N(0,0.3)$. The covariates $x=[x_1,x_2,x_3,x_4]^\intercal \in \mathbb{R}^4$ are sampled from $x \sim \mathcal{N}(\vect{0}_4,I_4)$ and $\mathcal{N}(\vect{1}_4,0.3^2I_4)$ for training and testing data, respectively. To construct segments, we set $\alpha_{0s} \in [-2,2]$, evenly spaced for $s \in \{1,2,\dots, 20\}$ and $\alpha_{1s} =-\alpha_{0s}\vect{1}_4$. 

We experiment with different sample sizes from $1000$ to $10000$ and report the average MSE with 100 repetitions for each. The results are summarized in \cref{fig: simulations_mr_erm}. In addition to the performance comparison summarized above, we observe in this example that the test loss of \MR and \DRSF significantly decreases as sample size becomes larger. This aligns our understanding that the differences in distributions across segments can be learned and adjusted for more effectively with more data. This also complies with the generalization bound in \cref{thm:multiple_robust_rate} which is tighter as $n_s$ becomes larger. However, the other the test losses for four methods does not exhibit a clear decreasing trend. 


\begin{figure}[ht]
\begin{center}
\centerline{\includegraphics[width=\columnwidth]{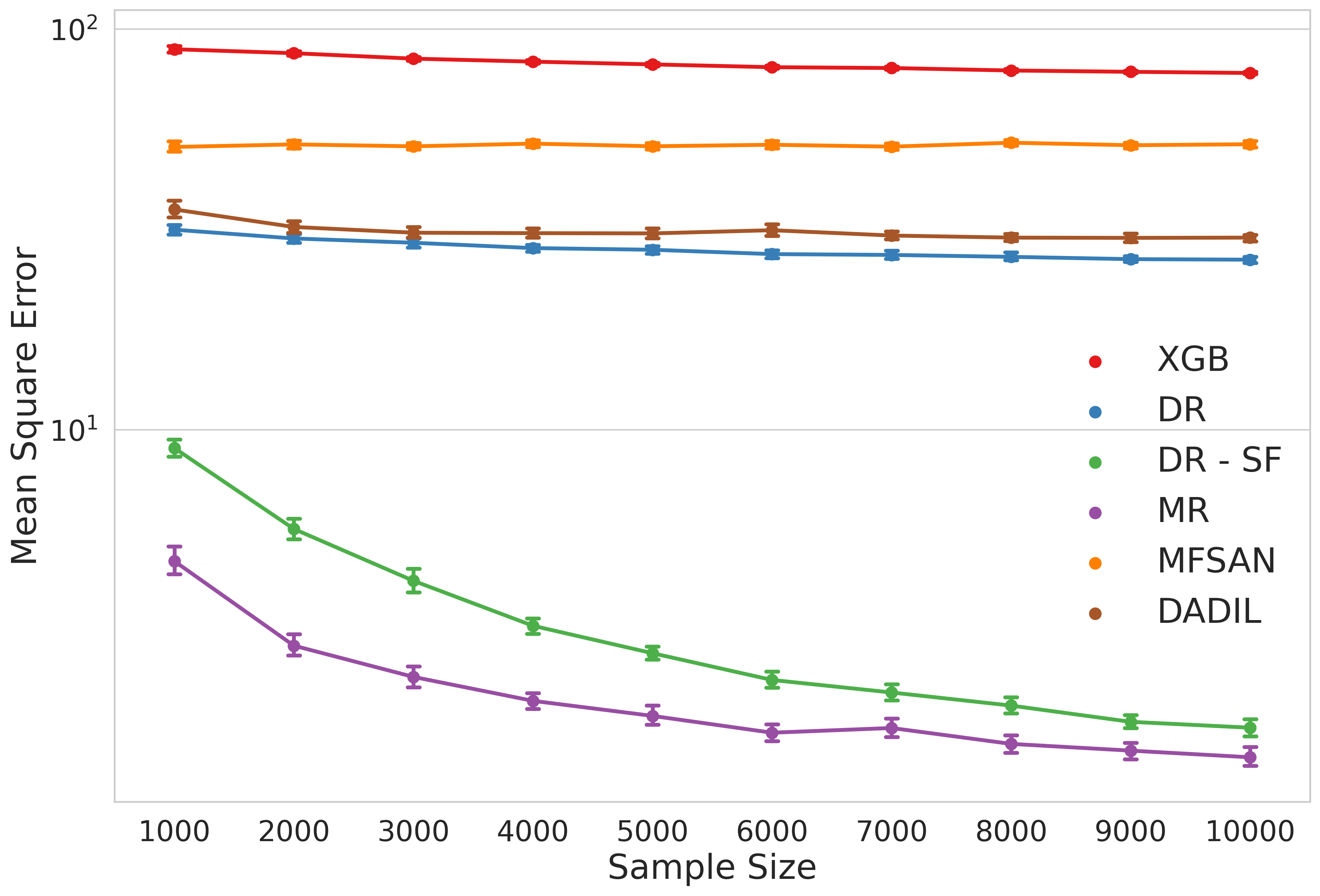}}
\caption{Performance of the proposed \MR method against the competitors in the simulation study. Error bars indicate 99\% confidence intervals.}
\label{fig: simulations_mr_erm}
\end{center}
\vskip -0.4in
\end{figure}

\subsection{Public Datasets}
\begin{table*}[t]

\begin{center}
\begin{small}
\caption{Performance of each of the methods on distribution shifted test sets of the UCI and Kaggle datasets.}
\label{tab: uci-test-metrics}
\begin{sc}
\begin{tabular}{lccccr} 
\toprule
 \multirow{2}{*}{MODELS}               & \multicolumn{2}{c}{BINARY CLASSIFICATION [CE]} & MULTICLASS [CE] & \multicolumn{2}{c}{REGRESSION [MSE]} \\ 
      & ADULT               & BANKING            & CUSTOMER        & APARTMENT         & ONLINE NEWS      \\
        \midrule
\XGB           & 1.000 (0.026)       & 1.0 (0.03)         & 1.0 (0.009)     & 1.0 (0.013)       & 1.0 (0.023)      \\
\DR     & 0.953 (0.027)       & 0.976 (0.03)       & 0.978 (0.009)   & 1.0 (0.013)       & 1.011 (0.023)    \\
\DRSF & 0.948 (0.027)       & 0.971 (0.03)       & 1.025 (0.01)    & 0.774 (0.01)      & 0.994 (0.023)    \\
\textbf{\MR}     & \textbf{0.901 (0.028)}       & \textbf{0.925 (0.03)}       & \textbf{0.834 (0.008)}   & \textbf{0.524 (0.008)}     & \textbf{0.949 (0.023)}    \\
\MFSAN         & 1.232 (0.028)       & 1.01 (0.032)       & 0.853 (0.008)   & 1.873 (0.02)      & 1.129 (0.023)    \\
\DADIL      & 3.285 (0.038)       & 1.347 (0.038)      & 1.679 (0.016)   & 2.895 (0.03)      & 1.121 (0.023)   \\\bottomrule
\end{tabular}

\end{sc}
\end{small}
\end{center}

\vskip -0.25in
\end{table*}
We consider two classification datasets (\texttt{Adult} and \texttt{Banking})  and two regression datasets (\texttt{Apartment} and \texttt{Online News}) from the UCI repository \cite{Dua:2019}. We also include a multiclass  \texttt{customer} segmentation dataset \cite{kaggle:2021} from Kaggle. 
We include summary statistics, the segments used for partitioning, and the process by which we construct and correct for dataset shifts in \cref{appendix: real data experiment details}. 

We report the test loss relative to \XGB for all of the methods in \cref{tab: uci-test-metrics}. We observe that \MR outperforms the competing models on all datasets, on binary, multi-classification and regression tasks with smallest standard errors on 4 out of 5 datasets. The results suggest the effectiveness of \MR in terms of both predictive power and robustness. We include the detailed test losses on all segments for all datasets in \cref{appendix: Model performance across segments}.

\subsection{Meta User City Prediction}
We consider a dataset utilized for user city prediction from Meta. The observational unit of the dataset is a (user, city) pair and the features are mostly based on the interaction between a user and a city. A user could have multiple candidate cities and thus be associated with multiple (user, city) pairs. The labels are collected from the users, based on responses to a survey. The responses $y\in\{0, 1\}$ indicate whether a given city is the primary location of a given user. These predictions are then fed through downstream applications for multiple use cases.

The labeled data consists of a curated training dataset and a test dataset, where there are approximately 50,000 users in the training set, 12,000 users in the test set and approximately 50 features. On average there are 5 candidate cities per user in the training set and 11.5 in the test set. As a result, the proportion of positive examples is 0.2 in the training set, compared to a lower proportion of 0.08 in the test set. The training set is derived from survey responses, which typically present fewer candidate cities. Consequently, a local label shift occurs due to the increased prevalence of negative examples in the test set, stemming from the larger pool of candidate cities per user. In this study, we segment users based on their country of origin. We consider a total of 21 segments, with 20 countries that are selected based on factors such as population size, user engagement level, and growth potential. The remaining users from other countries are grouped together as a single segment.




We assess each method's performance using test CE relative to the production model. We also include the \textbf{Brier score}, a measure of calibration of a classifier \cite{brier1950verification}. 
 As shown in \cref{tab: locations-test-metrics}, our multiply robust method consistently outperforms others in CE and Brier score. The CE results by country are also visualized in \cref{fig: relative_CE}, demonstrating improved performance compared to the production model across nearly all segments. 
 
\begin{table}[htb]
\begin{small}
\caption{Relative performance of various competing methods on location dataset with production model as baseline, and running time (in minutes) comparison. Both \MFSAN and \DADIL are implemented on GPU, whereas the remaining methods are run on CPU.}
\label{tab: locations-test-metrics}
\begin{center}
\begin{sc}
\begin{tabular}{lccc}
\toprule
\multirow{2}{*}{Models} & \multicolumn{2}{c}{Model Performance} & \multirow{2}{*}{Runtime}\\
 & CE & Brier & \\
\midrule
\XGB & 0.974 (0.012) & 0.978 (0.014) & 0.09\\
\DR & 0.93 (0.012) & 0.935 (0.013) & 0.26\\
\DRSF & 0.917 (0.011) & 0.927 (0.013) & 0.36\\
\textbf{\MR}& \textbf{0.852 (0.012)} & \textbf{0.864 (0.013)} & 5.17 \\
\MFSAN & 1.36 (0.015) & 1.401 (0.017) & 47.36\\
\DADIL & 1.879 (0.021) & 1.974 (0.019) & 59.86\\
\bottomrule
\end{tabular}
\end{sc}
\end{center}
\end{small}
\end{table}

\begin{figure}[ht!]
\begin{center}
\centerline{\includegraphics[width=\columnwidth]{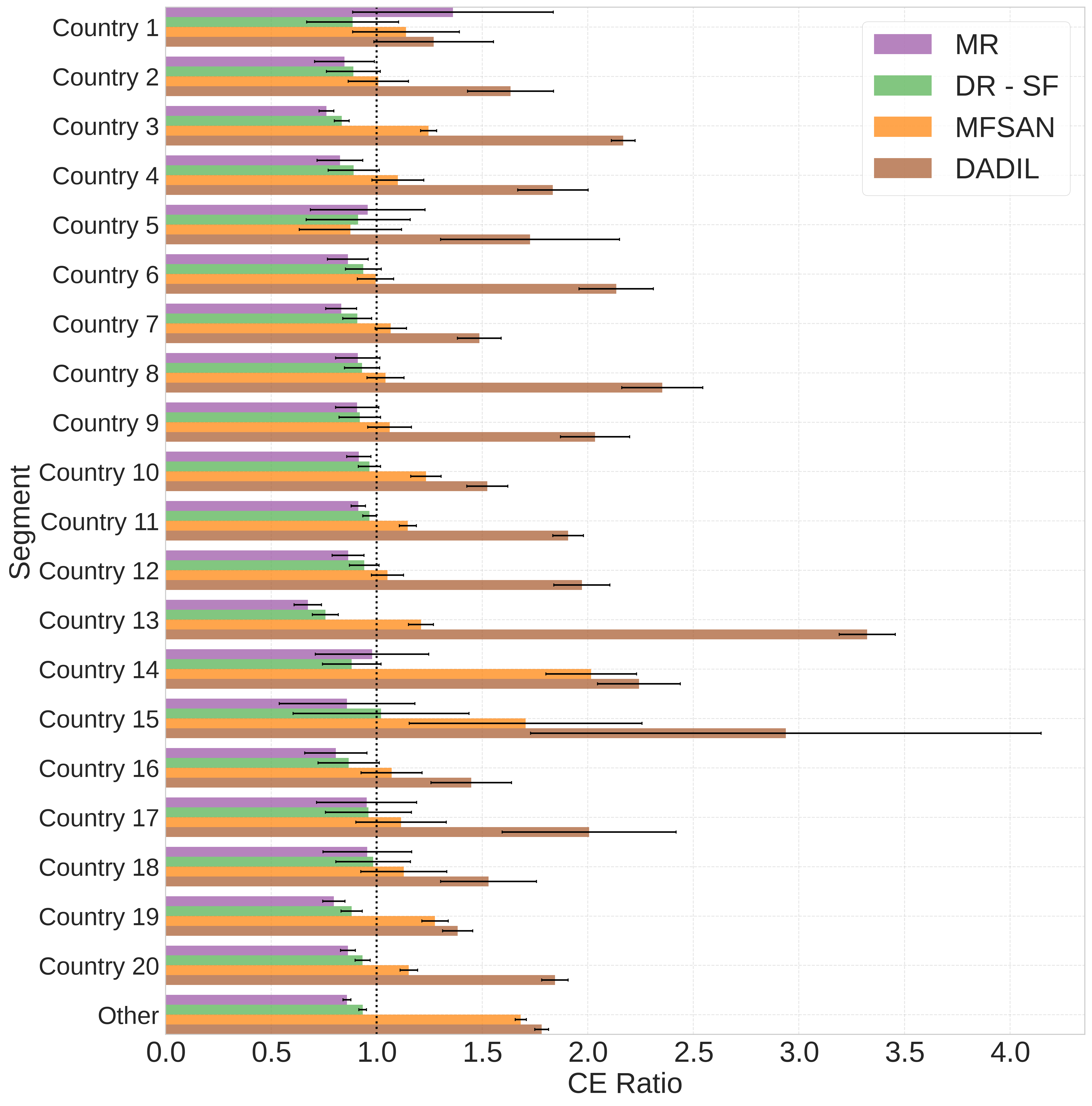}}
\caption{ Test CE relative to the production model (black dotted line) by country. Error bars denote 99\% confidence intervals.}
\label{fig: relative_CE}
\end{center}
\vskip -0.3in
\end{figure}

\section{Conclusion and Future Work}
\label{sec:conclusion}



Existing assumptions for domain adaptation can be overly restrictive when applied to an entire population. To address this, we introduce and study a domain adaptation framework that allows for the relaxation of classical distribution shift assumptions to hold only within segments of the data. This presents a statistical challenge in refining a predictor across all segments of our population. 

To address this challenge, we present a method to refine models across population segments. We provide theoretical justification and empirical evidence demonstrating the enhanced performance of our method compared to existing state-of-the-art domain adaptation techniques. Our approach is easily implementable using off-the-shelf algorithms. 






Potential avenues for future work include exploring methods of jointly learning each of the refined models, by incorporating segmentation clustering, treating this as a multitask learning problem. Extensions to image and text data could also be considered, possibly utilizing pre-trained models in place of the base models in \cref{alg:MR_estimator}. These problems however would involve further consideration of the trade-off between domain invariance properties and refinement for each segment. Lastly, our framework of studying a segmented distribution shift may have further connections with algorithmic fairness, as we avoid the problem of large segments dominating the training of the model. However, recent work explores the impact of distribution shift on fairness metrics and may be worth further exploration in segmented settings  \cite{jiang2024chasing}.

\section*{Acknowledgements}
The authors thank Sai Sumanth Miryala, Tao Wu, Xiaolu Xiong, and Andy Yu from the People Data Platform team for their support and feedback throughout this research.

\section*{Impact Statement}
This paper presents work whose goal is to advance the field of Machine Learning. There are many potential societal consequences of our work, none of which we feel must be specifically highlighted here. 





\bibliography{camera_ready}
\bibliographystyle{icml2024}
\newpage
\appendix
\onecolumn

\section{Additional Implementation Details}
\subsection{Methods for Learning Importance Weights}
\label{appendix:learning_weights}

We can apply standard techniques for learning the importance weights of the function class. Firstly, we consider learning covariate shift weights through discriminative learning, i.e. learning a probabilistic classifier that discerns between the weights \cite{bickel2009discriminative}. Let $T = \{0,1\}$ denote whether a sample is within the training or test distribution letting $P(\cdot) = \mathbb{P}(\cdot|T = 0)$ and $Q(\cdot) = \mathbb{P}(\cdot|T = 1)$, then the importance weights can be expressed up to proportion as

\begin{equation*}
    w(x) = \frac{dQ(x)}{dP(x)} \propto \frac{\mathbb{P}(T = 1| x)}{\mathbb{P}(T = 0| x)}.
\end{equation*}

Hence, one can estimate a probabilistic classifier to discriminate between $T = 0$ and $T = 1$, $ \mathbb{\widehat{P}}(T = 1| x)$, which can then be used to construct a plug-in estimate 
\begin{equation*}
    \widehat{w}(x) = \frac{\mathbb{\widehat{P}}(T = 1| x)}{\mathbb{\widehat{P}}(T = 0| x)}
\end{equation*}  

Secondly, we consider Penalized Risk Minimization \cite{nguyen2010estimating}. 
This method involves the estimation as follows:  
\begin{align*}
    \widehat w_s(x) &= \widehat g_s(x)  \\
    \widehat g_s &= \argmin_{g \in \mathcal{G}} \frac{1}{n_{s,tr}} \sum_{i \in N_s 1}g(X_i) \\
    &- \frac{1}{n_{s,test}} \sum_{j \in N'_{s} }^n\log(g(X'_j)) + \frac{\gamma_n}{2}\widetilde{I}^2(g)
\end{align*}
where $\mathcal{G}$ is a reproducing Kernel Hilbert space, $N_s$ and $N'_s$ refer to the indices corresponding to training and test indices in segment $s$ and $\widetilde{I}$ is a non-negative measure of complexity for $g$ such that $\widetilde{I}(w) < \infty$ (for example, $\widetilde{I}(g) = \norm{g}_{\calF}$). 

Another common approach to estimating importance weights is via Kernel Mean Matching (KMM) \cite{gretton2012kernel}. Given a positive definite kernel function $k(\cdot, \cdot) \to \mathbb{R}$ let $G_{ij} = k(X_i, X_j)$ and let $\kappa_i = \frac{n_{tr}}{n_{te}} \sum_{j = 1}^{n_{te}}k(X_i, X_j')$. Then kernel mean matching computes a set of weights $\{w_i\}_{i = 1}^{n_{tr}}$
\begin{align*}
    &\text{min}_{\vect{w}} \frac{1}{n_{tr}^2}\vect{w}^\intercal G \vect{w} - \frac{2}{n_{tr}^2} \kappa^\intercal \vect{w} \\
    \text{such that } &w_i \in [0,\eta] \text{ and } \bigg| \sum_{i = 1}^{n_{tr}} w_i - n_{tr} \bigg|\leq n_{tr}\epsilon 
\end{align*}
where a good choice of $\epsilon$ is $\mathcal{O}(\eta/\sqrt{n_{tr}})$. For further details see \cite{gretton2009covariate}. 

Lastly, when correcting for label shift one can use the Black Box Shift Estimation (BBSE) of \cite{lipton2018detecting}. This involves using a predictor $h(\cdot)$, and denote $ \widehat Y = h(X)$. The corresponding importance weights $\vect{\tilde w} \in \mathbb{R}^K$ where $\tilde w(y) = \frac{Q(y)}{P(y)}$, can be defined through a linear equation of the expected confusion matrix under the training distribution $\mathbf{C}$, where $\mathbf{C}_{ij} = P(\widehat Y = i,Y = j )$ and the predicted class frequencies on the test distribution $\vect{\mu}_{Q,h} \in \mathbb{R}^K$ where $\mu_{Q,h, j} = Q(\widehat Y = j)$. We differentiate the importance weights for each class $\vect{\tilde w} \in \mathbb{R}^{K}$ from $\vect{w}  \in \mathbb{R}^n$, the weight vector for each training observation using the tilde annotation. The importance weights can then be defined by the solution to the following: 

$$ \vect{\mu}_{Q,h} = \mathbf{C}\vect{\tilde w} $$

Using the training labels, one can estimate $\vect{\widehat \mu}_{Q,h}$ and $\mathbf{\widehat C}$ using empirical means. As long as $\mathbf{C}, \mathbf{\widehat{C}}$ are invertible, then one can estimate the importance weights of the class

$$ \vect{\widehat{\tilde{w}}} = \mathbf{\widehat{C}}^{-1}\vect{\widehat{\mu}}_{Q,h} $$

\section{Proofs of Theorem~\ref{thm:multiple_robust_rate} and supporting Lemmas}

In order to prove Theorem~\ref{thm:multiple_robust_rate} we first introduce a series of supporting lemmas, after which we proceed with the proof of the theorem in Appendix~\ref{sec: Main Theorem Proof}.

\subsection{Supporting Lemmas} \label{sec: supporting lemmas}
We first introduce a useful Lemma for proving the main theorem of our paper. 

\begin{lemma}[Best Linear Combination As A Prior] \label{lem:linear_combination}
    Let $\{h_m\}_{m = 1}^M$ be a collection of base models contained in some function space $h_m \in \calF$. Let $s$ denote the target segment of interest, and let $N_s \subset \{1,2,n_{tr}\}$ be the set of indices of the training data corresponding to segment $s$ with $|N_s| = n_s$. Let $\beta \in \mathcal{B} \subset \mathbb{R}^M$  denote the parameters for a linear combination of $h_m$. Let $\mathcal{B}$ be a convex set in the parameter space. Let $\beta^{*(s)}$ denote the optimal linear combination satisfying $\beta^{*(s)} = \argmin_{\beta \in \mathcal{B}} R_P[\vect{h}^\intercal\beta|s]$ and the empirical estimate $\widehat \beta^{(s)} = \argmin_{\beta \in \mathcal{B}} R[\vect{h}^\intercal\beta|\mathbf{Y}, \mathbf{X}, s]$. 
    Denote the following random variables and parameters derived from the loss function and the base predictors:  
    \begin{align*}
        Z_{mi} &= h_m(X_i),  \quad  H = \E[\nabla^2_{\beta} \ell(Y,Z^\intercal\beta^{*(s)})|s] \\
         G &= \E[\nabla_{\beta} \ell(Y,Z^\intercal\beta^{*(s)}) \nabla_{\beta} \ell(Y,Z^\intercal\beta^{*(s)})^\intercal |s] \\
        M_{eff} &= \text{Tr}\left( H^{-1} G\right), \quad M_F = \max_{m} \norm{h_m}_{\calF} \\
        \Sigma &= \E[ZZ^\intercal|s], \quad \text{Where } Y_i, X_i \sim_{iid} P^{(s)}(y,x).
    \end{align*}

    Assume that the conditions of \cref{assumption: regularity conditions} hold. 
    Then for all $n \geq \rho M_{eff} M \log(e M / \delta)$ 
    \begin{equation}
        \norm{\vect{h}^\intercal\beta^{*(s)} - \widehat \beta^{(s) \intercal} \vect{h}}_{\calF} \leq C\frac{M_F}{\tilde{\lambda}_1}\sqrt{\frac{MM_{eff} \log(e/\delta)}{n_s}}
    \end{equation}
    For a universal constant $C$
\end{lemma}
Lemma~\ref{lem:linear_combination} follows from a simple application of the results from \cite{ostrovskii2021finite}. The main idea we exploit is that learning a simple linear combination is easier than learning in the larger function class $\calF$. Since we assume that the total number of base estimators is not increasing, then $M_{eff}$, $M_{F}, M$ and $\tilde{\lambda}_1$ are all constants and we recover the stage 1 model at a rate of $n_s^{-1/2}$. This therefore, allows us to learn this prior faster than that of the stage 2 model.


\subsection{Proof of Lemma~\ref{lem:linear_combination}}
\begin{proof}
    Recall $h_m: \X \to \Y$ are the set of base-models, and let $\vect{h}$ denote the vector of such functions. Denote a linear combination of these with weights $\beta$ such that $\vect{h}^\intercal\beta = \sum_{m = 1}^M h_m\beta_m$. 

The first step is an immediate application of Theorem 1.1 of \cite{ostrovskii2021finite}. This theorem is applied as follows. Let $\rho$ denote the minimum value such that $\Sigma \preceq \rho  H$. Then for all $n_s \geq \rho M_{eff} M \log(e M / \delta)$ 

\begin{equation*}
    R[\vect{h}^\intercal \widehat \beta^{(s)}|s] - R[\vect{h}^\intercal\beta^{*(s)}|s] \lesssim (\widehat \beta^{(s)} - \beta^{*(s)})^\intercal H^2(\widehat \beta^{(s)} - \beta^{*(s)}) \lesssim \frac{M_{eff} \log(e/\delta)}{n_s} \label{eq: Ostrovskii Bach Bound}
\end{equation*}
where $\lesssim $ denotes less than or equal to, up to a universal constant $C$. 

Further let $\tilde{\lambda}_1$ denote the smallest eigenvalue of $ H$. Then with probability at least $1 - \delta$

\begin{equation*}
   \norm{\widehat \beta^{(s)} - \beta^{*(s)}}_2^2 \lesssim \frac{M_{eff} \log(e/\delta)}{\lambda^2_1 n_s}. 
\end{equation*}

Therefore, for any norm associated with a function space $\norm{\cdot}_{\calF}$, 

\begin{align*}
    \norm{\vect{h}^\intercal\beta^{*(s)} - \vect{h}^\intercal \widehat \beta^{(s)}}_{\calF} &\leq \sum_{m = 1}^M \norm{h_m(\beta_m^{*(s)} - \widehat \beta_r)}_{\calF} \\
    &= \sum_{m = 1}^M \norm{h_m(\beta_m^{*(s)} - \widehat \beta_r)}_{\calF} \\
    &\leq \norm{\widehat \beta^{(s)} - \beta^{*(s)}}_1 \sup_{m \in \{1,2,\dots, M\}}\norm{h_m}_{\calF}
\end{align*}

Therefore our proof is complete after applying $\norm{\widehat \beta^{(s)} - \beta^{*(s)}}_1 \leq \sqrt{M} \norm{\widehat \beta^{(s)} - \beta^{*(s)}}_2$. 
\begin{equation*}
    \norm{\vect{h}^\intercal\beta^{*(s)} - \vect{h}^\intercal \widehat \beta}_{\calF} \lesssim \frac{M_F}{\tilde{\lambda}_1}\sqrt{\frac{MM_{eff} \log(e/\delta)}{n_s}}
\end{equation*}
\end{proof}

We next consider the following lemma, first introduced in  \cite{reddi2015doubly} with its few line proof for completeness.
\begin{lemma}[Lemma 2 of \cite{reddi2015doubly}]\label{lem:Rademacher_complexity}
    Suppose $f^*_{q,\lambda_{q,s}}$ denote the population maximizer as in equation~\eqref{eq:population_shift_solution}. The following bound holds with probability at least $1 - \delta$
    \begin{align*}
        \sup_{f: \norm{f - f_0} \leq \nu} \left| R[f|\mathbf{Y}, \mathbf{X},s;w_s] - R_Q[f|s] \right| &\leq \eta L \left(\frac{2\nu}{n_s}\sqrt{\text{tr}(\vect{K}_s)} + 3\sqrt{\frac{\log(2/\delta)}{n_s}} + \frac{1}{n_s} \sqrt{\sum_{i \in N_s} f_0^2(X_i)}\right) 
    \end{align*}
\end{lemma}

\begin{proof}    
    This proof simply uses the standard Rademacher complexity bounds for kernel methods. 
    
    Let $\mathcal{\widehat R}_{n_s}(\mathcal{F})$ denote the empirical Rademacher average of a function class $\mathcal{F}$

    $$ \mathcal{\widehat R}_{n_s}(\mathcal{F}) = \mathbb{E}\left[ \sup_{f \in \mathcal{F}} \frac{2}{n_s}\left |\sum_{i \in N_s} \sigma_i f(x_i) \right |  \right]$$
    where $\sigma_i$ are independently drawn uniform $\{\pm 1\}$ random variables.
    
    Using Lemma 22 of \cite{bartlett2002rademacher}, and in combination with the boundedness of the loss function. 
    It can be shown that with probability $1 - \delta$
    \begin{align*}
        \sup_{f \in \mathcal{F}} |R[f|\mathbf{Y}, \mathbf{X}, s] - R_P[f|s]| &\leq  L\mathcal{\widehat R}_n(\mathcal{F}) + 3L\sqrt{\frac{\log(2/\delta)}{2n}}
    \end{align*}
    due to the $L$-Lipschitz and $L$-bounded loss function $\ell$
    Furthermore, in this setting, if the empirical distributions are weighted by some function $w_s(y,x) \leq \eta$. Then 
    \begin{align*}
        \sup_{f \in \mathcal{F}} |R[f|\mathbf{Y}, \mathbf{X}, w_s] - R_P[f]| &\leq L\eta\mathcal{\widehat R}_n(\mathcal{F}) + 3L\eta\sqrt{\frac{\log(2/\delta)}{2n}}
    \end{align*}
    If $\mathcal{F} = \{f: \norm{f - f_0} \leq \nu\}$, then $\mathcal{\widehat R}_n(\mathcal{F}) = \frac{2\nu}{n}\sqrt{\text{tr}(\vect{K})} + \frac{1}{n}\sqrt{\sum_{i = 1}^n f_0^2(X_i)} $ and the proof is complete. 
\end{proof}

Lastly, we include an upper bound on the weights.
\begin{lemma}\label{lem:importance_weight}
    Let $\widehat w_s$ be an estimator of the true importance weights. 
    Then
    \begin{equation*}
        \sup_{f \in \calF}|R[f|\mathbf{Y}, \mathbf{X},s, \widehat w_s] - R[f|\mathbf{Y}, \mathbf{X},s, w_s]| \leq L \frac{\norm{\widehat w_s - w_s}_1}{n_s}
    \end{equation*}
\end{lemma}
\begin{proof}
    The proof is a straightforward application of Holder's inequality. 
    \begin{align*}
        \sup_{f \in \calF}|R[f|\mathbf{Y}, \mathbf{X},s, \widehat w_s] - R[f|\mathbf{Y}, \mathbf{X},s, w_s]| &= \sup_{f \in \calF}|\frac{1}{n_s}\sum_{i \in N_s} \ell(Y_i,f(X_i))(\widehat w_s(X_i,Y_i) - w_s(X_i,Y_i))| \\
        &\leq \sup_{f \in \calF}\text{max}_{i \in N_s}|\ell(Y_i,f(X_i))|\frac{\norm{\widehat w_s - w_s}_1}{n_s} \\
        &\leq L\frac{\norm{\widehat w_s - w_s}_1}{n_s}
    \end{align*}
\end{proof}

\subsection{Proof of Theorem~\ref{thm:multiple_robust_rate}} \label{sec: Main Theorem Proof}
In this proof, we will rely on several lemmas, which we include subsequently in \cref{sec: supporting lemmas}. 

\begin{proof}
Recall the two-stage multiply robust estimator
\begin{align*}
    \widehat f_{MR}^{(s)} &= \argmin_{f \in \mathcal{F}} R[f|\mathbf{Y}, \mathbf{X}, s; \widehat w_s ] + \lambda' \norm{\vect{h}^\intercal\widehat \beta^{(s)} -  f}_{\calF} \\
    &= \argmin_{f: \norm{\vect{h}^\intercal\widehat \beta^{(s)} -  f}_{\calF} \leq \nu'} R[f|\mathbf{Y}, \mathbf{X}, s;  \widehat w_s ]
\end{align*}

where the constraint value $\nu'$ denotes the corresponding value to the regularization parameter $\lambda'$. 

We first draw on  \cref{lem:linear_combination} which states with probability at least $1 - \delta$,   $\norm{\vect{h}^\intercal\widehat \beta^{(s)} -  f}_{\calF} \leq  \norm{\beta^{*(s) \intercal}  \vect{h} -  f}_{\calF}  + C\frac{M_F}{\tilde{\lambda}_1}\sqrt{\frac{RM_{eff} \log(e/\delta)}{n_s}} = \nu'$

Next, let $B$ denote the bound from Lemma~\ref{lem:Rademacher_complexity} which holds with probability at least $ 1 - \delta$
\begin{align*}
    \sup_{f: \norm{f - \vect{h}^\intercal\widehat \beta} \leq \nu'} \left| R[f|\mathbf{Y}, \mathbf{X},s;w_s] - R_Q[f|s] \right| &\leq B(n_s, \delta) \\
    B(n_s, \delta) &:= \eta L \left(\frac{2\nu'}{n_s}\sqrt{\text{tr}(\vect{K}_s)} + 3\sqrt{\frac{\log(2/\delta)}{n_s}} + \frac{1}{n_s} \sqrt{\sum_{i \in N_s} (\vect{h}^\intercal\widehat \beta)^2(X_i)}\right)
\end{align*}

Furthermore, let $C(n_s, \widehat \beta)$ denote the bound from Lemma~\ref{lem:importance_weight} denoting the statistical error from estimating the importance weights:
$$ |R_Q[f|\mathbf{Y}, \mathbf{X},s, \widehat w_s] - R_Q[f|\mathbf{Y}, \mathbf{X},s,w_s]| \leq C(n_s, \widehat w_s) := L\frac{\norm{\widehat w_s - w_s}_1}{n_s}$$

Using these two bounds, we can define a simple bound on the generalization error, with probability at least $1 - \delta$. 
\begin{align*}
    R_Q[\widehat f_{MR}^{(s)}|s] &\leq R_Q[\widehat f_{MR}^{(s)}|\mathbf{Y}, \mathbf{X},s, w_s] + B(n_s, \delta) \\
    &\leq R[\widehat f_{MR}^{(s)}|\mathbf{Y}, \mathbf{X},s, \widehat w_s] + B(n_s, \delta) + C(n_s, \widehat w_s)\\
    &\leq R[f^{*(s)}_{\nu_{q,s}}|\mathbf{Y}, \mathbf{X},s, \widehat w_s] + B(n_s, \delta) + C(n_s, \widehat w_s) \\
    &\leq R[f^{*(s)}_{\nu_{q,s}}|\mathbf{Y}, \mathbf{X},s, w_s] + B(n_s, \delta) + 2C(n_s, \widehat w_s) \\
    &\leq R[f^{*(s)}_{\nu_{q,s}}|s] + 2B(n_s, \delta) + 2C(n_s, \widehat w_s) \\
\end{align*}

The third inequality follows from the fact $R[\widehat f_{MR}^{(s)}|\mathbf{Y}, \mathbf{X},s, \widehat w_s] \leq R[f^{*(s)}_{\nu_{q,s}}|\mathbf{Y}, \mathbf{X},s, \widehat w_s] $

Therefore, with probability at least $1 - 2\delta$
\begin{equation}
    R_Q[\widehat f_{MR}^{(s)}|s] \leq R[f^{*(s)}_{\nu_{q,s}}|s] + 2\eta L \left(\frac{2\nu'}{n_s}\sqrt{\text{tr}(\vect{K}_s)} + 3\sqrt{\frac{\log(2/\delta)}{n_s}} + \frac{1}{n_s} \sqrt{\sum_{i \in N_s} f_0^2(X_i)}\right) + 2L\frac{\norm{\widehat w_s - w_s}_1}{n_s}
\end{equation}

Where 
\begin{equation*}
    \nu' = \norm{\vect{h}^\intercal\beta^{*(s)} - f^{*(s)}_{\nu_{q,s}}} + C\frac{M_F}{\tilde{\lambda}_1}\sqrt{\frac{M M_{eff} \log(e/\delta)}{n_s}}
\end{equation*}
for a universal constant $C$. 

In order for our method to be effective  $\nu_{MR} = \norm{\vect{h}^\intercal\beta^{*(s)} -  f^{*(s)}_{\nu_{q,s}}}_{{\cal F}}  \ll \norm{f^{*(s)}_{\nu_{q,s}}}_{{\cal F}} = \nu_{q,s}$, however this is not necessary for developing the bound in the first place. In other words, the best linear combination of the base models trained on the unweighted segment, should be close to the optimizer of the weighted set in the function class $\calF$. The idea here is that if we have a candidate set of prior functions, then finding the best linear combination is generally going to be a much easier statistical task than learning a function class regularized towards $0$. 
\end{proof}

\subsection{Bounds on estimating the importance weights} \label{appendix: importance weights estimation bounds}

Our upper bound in Theorem~\ref{thm:multiple_robust_rate} relies on estimating the importance weights $\norm{\widehat w_s - w_s}_1$. We focus on two upper bounds, one for covariate shift and one for label shift. 

\begin{lemma}[Lemma 8 of \cite{reddi2015doubly}]
    Suppose that $\widehat w_s$ is an estimator for $w_s$ obtained by Penalized Risk Minimization (as in \cite{nguyen2010estimating}). 
    Suppose $\gamma_n = cn_s^{-2/(2 + \tau)}$ for some $\tau > 0$. Let $\eta : w_s \leq \eta$ denote the bound of the magnitude of the true importance weights. Then with probability at least $1 - \delta$ 
    \begin{equation}
        \frac{1}{n_s} \norm{\widehat w_s - w_s}_1 \leq \sqrt{\gamma_n \eta} + \eta \sqrt[\leftroot{-2}\uproot{2}4]{\frac{8}{n_s}\log(2/\delta)}
    \end{equation}
\end{lemma}

We note that in general, learning the importance weights under covariate shift non-parametrically is a statistically challenging task. An alternative approach may be to propose some parametric model for $w_s$, which can help to reduce the final complexity.

A second approach we consider is the black box shift estimation (BBSE) for label shift, where in this case the outcome $y \in \{1,2,\dots, K\}$ \cite{lipton2018detecting}. It can be shown that the importance weights can be estimated at a $n^{-1/2}$ using the output of a classifier $\widehat y = h(x)$. For simplicity, assume that there are an equal number of train and test samples $n_{train} = n_{test} = n$. 

\begin{theorem}[Theorem 3 from \cite{lipton2018detecting}]
Let $\mathbf{C}$ denote the expected confusion matrix computed under the training distribution $\mathbf{C}_{i, j} = P(\widehat Y = i, Y = j)$ where $\mathbf{C}$ is invertible where $\sigma_{min}$ it's smallest eigenvalue. Then there is a universal constant $C$ such that for all $n > 80 \log(n)\sigma_{min}^{-2} $. Then with probability at least $1 - \frac{5 K }{n^{10}}$. 
\begin{align*}
    \sum_{y \in \{1,2,\dots, K\}} (\widehat{\tilde{w}}_s(y) - \tilde{w}_s(y) )^2 &\leq \frac{C}{\sigma^2_{min}}\left(\frac{(\sum_{y} (\tilde w_s(y) )^2 + K)\log(n)}{n} \right)
\end{align*}
\end{theorem}

It follows from a simple application of the Cauchy-Schwarz inequality that we can upper bound $\norm{\widehat w_s - w_s}_1/n$

\begin{align*}
    \norm{\widehat w_s - w_s}_1/n &= \frac{1}{n}\sum_{i \in N} |\widehat{\tilde{w}}_s(y_i) - \tilde{w}_s(y_i)| \\
    &= \sum_{y \in \{1,2,\dots, K\}} \frac{n_k}{n} |\widehat{\tilde{w}}_s(y) - \tilde{w}_s(y)| \\
    &\leq \sqrt{\sum_{y \in \{1,2,\dots, K\}} (\frac{n_k}{n})^2 \sum_{y \in \{1,2,\dots, K\}} (\widehat{\tilde{w}}_s(y) - \tilde{w}_s(y))^2} \\
    &\leq \sqrt{\sum_{y \in \{1,2,\dots, K\}} (\widehat{\tilde{w}}_s(y) - \tilde{w}_s(y) )^2} \\
    &\leq \frac{\sqrt{C}}{\sigma_{min}\left(\sum_{y}\tilde w^2(y) + K\right)}\sqrt{\frac{\log(n)}{n}}
\end{align*}
where $n_k$ is the number of observations in $N$ with outcome $y = k$. Therefore, we can effectively learn the weights. 

\section{Additional Empirical Details} \label{appendix: empirical_results}

\subsection{Data splitting and cross validation}\label{sec: cross_validation}
The primary reason for data splitting is to avoid overfitting that can occur when the same data is used for training the base models ${h_m}_{m = 1}^M$ and subsequently fitting the stage 1 estimator. Alternatively, one may split the training data and preserve a holdout set (e.g. 20\%) which will be used for estimating $\beta$, then use use the fine-tuning dataset to estimate $\widehat f^{(s)}$. We provide a visualization in \cref{fig:mr_cross_validation}. This splitting procedure can be simply generalized to a K-fold versions, as well as grouped versions for data types that are clustered by observational unit (for example, the multiple candidate cities for each userid in our location dataset). Cross-validation details for the numerical experiments are included in \cref{appendix: real data experiment details} and \cref{appendix: tuning_parameters}. 

Given a training and test dataset, we constructed a cross-validation procedure as follows: 
\begin{enumerate}
    \item Split the training data into $K$ folds.
    \item For each $k$ in ${1,2, … K}$:
    \begin{enumerate}
        \item Use the $k^{th}$ fold as the validation fold, and combine the remaining folds to form the train fold.
        \item Fit the MR estimator using algorithm 1 on the train fold, splitting it in a base train and a tune fold in the process.
        \item Compute the importance weights on the validation fold to match the test data.
        \item For each segment, evaluate the fine-tuned MR estimator on the weighted validation fold.
    \end{enumerate}
\end{enumerate}

\begin{figure}[ht]
\vskip -0.1in
\begin{center}
\centerline{\includegraphics[width=0.6\columnwidth]{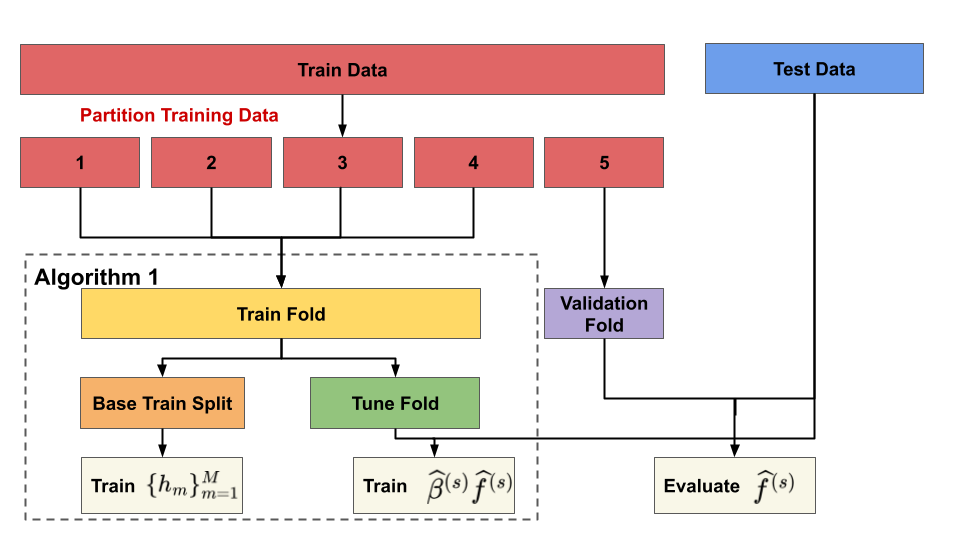}}
\caption{ Data splitting for the multiply robust estimator. We use a holdout set for refining the linear combination, and train the second stage model on the whole training fold segment.}
\label{fig:mr_cross_validation}
\end{center}
\vskip -0.4in
\end{figure}

\subsection{Real data experiment details and constructed distribution shifts. } \label{appendix: real data experiment details}

We include a brief set of summary statistics for each of the datasets in our application along with descriptions of the constructed shifts in \cref{tab: dataset -summary}. We include the \texttt{Apartment} price and \texttt{Online News} popularity datasets for regression,  \texttt{Banking} marketing campaign and \texttt{Adult} census datasets for binary classification, and the \texttt{customer} segmentation dataset \cite{kaggle:2021} for multi-classification.

\begin{table}[htb]
\caption{Summary statistics for UCI datasets}
\label{tab: dataset -summary}
\vskip 0.15in
\begin{center}
\begin{small}
\begin{sc}
\begin{tabular}{lcr}
\toprule
Dataset     & Sample Size              & Number of Covariates \\ \midrule
Adult       & 32561                    & 15                   \\
Banking     & 45211                    & 16                   \\
Customer    & 10695                    & 8                   \\ 
Apartment   & 87257                    & 6                    \\
Online News & 39644                    & 51                   \\ \midrule
Dataset     & Distribution Shift       & Segments                 \\ \midrule
Adult       & Lab. Shift               & Work Class (7)           \\
Banking     & Lab. Shift               & Occupation (12)          \\
Customer    & Lab. Shift               & Anonymous Group (8)      \\ 
Apartment   & Cov. Shift               & States (30)              \\
Online News & Cov. Shift               & Day of Week (7)          \\ \midrule
Dataset     & Task                     & Outcome (Positivity Rate)           \\ \midrule
Adult       & Class.             & Income $\geq$\$50K (0.24)           \\
Banking     & Class.             & Subscription (0.116)                \\
Customer    & MultiClass.        & \begin{tabular}{r} Consumer Group [A,B,C,D] \\ (0.244, 0.230, 0.244, 0.281) \end{tabular} \\
Apartment   & Reg.               & Price (\$)                          \\
Online News & Reg.               & log(Shares)                         \\ \bottomrule
\end{tabular}
\end{sc}
\end{small}
\end{center}
\vskip -0.1in
\end{table}

In the regression datasets, we construct a covariate shift by picking the feature most correlated with the outcome, and creating a propensity for sampling into the train and test set respectively. In the Apartment price dataset, listings over $2000$ square feet are selected into the test set with probability $0.8$ with the remaining having probability $0.2$. Similarly, if the average number of shares for articles by articles with the same key words is greater than $6000$, then we select these into the test set with probability $0.8$, and otherwise $0.2$ respectively. In the case of label shift, we first construct an $80\%$, $20\%$ split into a train and test set, then followup by sub-sampling the positive cases in the test set by $1/2$, thereby creating a change in the rate of positive classes in the test set. For the multiclass dataset, we resample the test fold according to the rates $(0.4,0.1,0.1,0.4)$ for each of the $4$ classes respectively.

For the experiments on the location platform dataset, we refer to the set of parameters in \cref{appendix: empirical_results}, otherwise we refer to the default parameters of the models in \cref{appendix: empirical_results}. All cross-validation procedures are $K = 5$ fold. 

When correcting for covariate shift, we consider learning a Logistic Regression classifier, then use the predicted odds ratio as the importance weights. We this approach as it is simple to scale to the dataset sizes we consider as opposed to KMM or PRM which scale quadratically with training sample size in their complexity. When correcting for label shift, we use an \XGB classifier trained on the unweighted training set as a base model and apply the BBSE method for label shift \cite{lipton2018detecting}. 

\subsection{Computing Environment and Tuning Parameters} \label{appendix: tuning_parameters}
\subsubsection{Computing Environment}
Both \MFSAN and \DADIL are implemented on NVIDIA Tesla P100-SXM2 GPUs with 16 GB of memory. All other methods are run on Intel Core Processor (Broadwell) Model 61 with 36 cores.
\subsubsection{Tuning Parameters}
\textbf{\XGB Tuning Parameters} \\
When cross validation is not used to pick the tuning parameters we use the default parameters for \XGB. The default for the refinement step stacked \XGB are: \texttt{ max\_depth: 2, n\_estimators: 25}.

Tuning parameters for the clustered segments are by default set to \texttt{colsample\_bytree: 1.0, learning\_rate: 0.1, max\_depth: 3, n\_estimators: 200,  subsample: 0.8} while the base model is trained with the default parameters. 

\textbf{Cross Validation Tuning Parameters} \\
All experiments involving \XGB for base models were tuned with the following \texttt{n\_estimators:} [10, 25, 50, 200, 300, 500], \texttt{max\_depth}: [2, 3,  5, 7], \texttt{subsample}: [0.8, 1.0], \texttt{colsample\_bytree}: [0.8, 1.0]. The \texttt{learning\_rate} was kept at 0.1. For second stage estimators (refinement), we grid search over \texttt{learning\_rate}: [0.001,0.01,0.1], \texttt{n\_estimators:} [0, 10, 25, 50], \texttt{max\_depth}: [0, 1, 3], \texttt{subsample}: [0.8, 1.0], \texttt{colsample\_bytree}: [0.8, 1.0]. All methods were tuned using $5$ fold cross-validation. 

\textbf{Tuning parameters for \MFSAN}
The parameters of \MFSAN were selected via grid search, while the search space are defined as \texttt{batch\_size}: [32, 64, 128, 256], \texttt{learning\_rate}: [0.0001, 0.0005, 0.001, 0.005], \texttt{adaptation\_factor}: [0.1, 0.2, 0.3, 0.4, 0.5], \texttt{n\_iterations}: [200, 300, 400, 500]. We uses \textit{adam} optimizer throughout the experiments, and adapted the size of the sub-networks within \MFSAN by different sizes of datasets to avoid overfitting. 

\textbf{Tuning parameters for \DADIL} 
The parameters of \DADIL were selected via grid search, while the search space are defined as \texttt{batch\_size}: [64, 128, 256], \texttt{learning\_rate}: [0.001, 0.01, 0.1, 0.5],  \texttt{n\_components}: [6, 7, 8, 9, 10], \texttt{n\_iterations}: [100, 150, 200]. We uses \textit{adam} optimizer throughout the experiments. 

\subsection{Learned Clusters by Segment} 

After computing a distance matrix using \cref{alg:segment_clustering} we apply Ward's linkage for hierarchical agglomerative clustering and select the threshold that maximizes the number of base models $M$ before isolating a single segment.

\begin{table}[ht!]
\caption{Segment Clusters of the \texttt{Adult} Dataset}
\begin{center}
\begin{small}
\begin{sc}
\begin{tabular}{lr}
\toprule
\textbf{Cluster}                     & \textbf{Work Class}   \\
\midrule
0                                    & {[}`Other' `Private' `Self-emp-inc'{]}        \\
1                                    & {[}`Federal-gov' `Local-gov' `Never-worked' `Self-emp-not-inc'{]} \\
\bottomrule
\end{tabular}
\end{sc}
\end{small}
\end{center}
\end{table}

\begin{table}[ht!]
\caption{Segment Clusters of the \texttt{Banking} Dataset}
\begin{center}
\begin{small}
\begin{sc}
\begin{tabular}{lr}
\toprule
\textbf{Cluster}                     & \textbf{Job Category}   \\
\midrule
0                                    & {[}`admin.' `blue-collar' `services' `student'{]}        \\
1                                    & {[}`entrepreneur' `management' `self-employed' `technician' `unemployed'{]} \\
2                                    & {[}`housemaid' `retired' `unknown'{]}                    \\
\bottomrule
\end{tabular}
\end{sc}
\end{small}
\end{center}
\end{table}

\begin{table}[ht!]
\caption{Segment Clusters of the \texttt{Apartment} Dataset}
\begin{center}
\begin{small}
\begin{sc}
\begin{tabular}{lr}
\toprule
\textbf{Cluster}                     & \textbf{Country}   \\
\midrule
0                                    & {[}`AZ' `FL' `GA' `NC' `NV' `TN' `TX' `VA'{]}        \\
1                                    & {[}'CA' `MA' `NJ'{]} \\
2                                    & {[}`CO' `IL' `MD' `OH' `PA' `WA' {]}                    \\
3                                    & {[} 'LA' `MO' `NE' {]}                    \\
\bottomrule
\end{tabular}
\end{sc}
\end{small}
\end{center}
\end{table}

\begin{table}[ht!]
\caption{Segment Clusters of the \texttt{Online News} Dataset}
\begin{center}
\begin{small}
\begin{sc}
\begin{tabular}{lr}
\toprule
\textbf{Cluster}                     & \textbf{Day of the Week}   \\
\midrule
0                                    & {[}`Saturday' `Sunday'{]}        \\
1                                    & {[}`Monday' `Tuesday' `Wednesday' `Thursday' `Friday'{]} \\
\bottomrule
\end{tabular}
\end{sc}
\end{small}
\end{center}
\end{table}

\begin{table}[ht!]
\caption{Segment Clusters of the \texttt{Customer} Dataset}
\begin{center}
\begin{small}
\begin{sc}
\begin{tabular}{lr}
\toprule
\textbf{Cluster}                     & \textbf{Anonymized Category}   \\
\midrule
0                                    & {[}`Cat\_1' `Cat\_3' `Cat\_5' `Cat\_6' `Cat\_7' `Unknown'{]}        \\
1                                    & {[}`Cat\_2' `Cat\_4'{]} \\
\bottomrule
\end{tabular}
\end{sc}
\end{small}
\end{center}
\end{table}

\begin{table}[ht!]
\caption{Base Model Clustering of Location Platform Dataset}
\begin{center}
\begin{small}
\begin{sc}
\begin{tabular}{lr}
\toprule
\textbf{Cluster}                     & \textbf{Countries}   \\
\midrule
0                                    & {[}`Country 1' `Country 7' `Country 12' `Country 18' `other'{]}        \\
1                                    & {[}`Country 2' `Country 3' `Country 6' `Country 13' `Country 16' `Country 17' `Country 20'{]} \\
2                                    & {[}`Country 4' `Country 9' `Country 14'{]}                     \\
3                                    & {[}`Country 5' `Country 8' `Country 10' `Country 11'{]}                \\
4                                    & {[}`Country 15' `Country 19'{]}                         \\
\bottomrule
\end{tabular}
\end{sc}
\end{small}
\end{center}
\end{table}

\newpage 
\subsection{Model Performance Across Segments} \label{appendix: Model performance across segments}

We include the performance of each of the models across the segments considered for each of the real datasets. We bold the top-performing model on each of the segments. We also calculated the area under the ROC curve, area under the precision-recall curve, and the F1 score for classification problems, however, their results closely align with production or unweighted models and are thus omitted for brevity. For classification tasks, Brier scores gave similar results to CE and are also omitted to avoid repetition. 


\begin{table}[ht!]
\caption{Apartment Prices. Partition by State}
\begin{center}
\begin{small}
\begin{sc}
\begin{tabular}{lcccccc}
\toprule
Segment & \XGB & \DR & \DRSF & \MR & \MFSAN & \DADIL \\
\midrule
AZ & 1.0 (0.062) & 1.004 (0.063) & \textbf{0.378 (0.032)} & 0.392 (0.03) & 3.404 (0.173) & 3.801 (0.167) \\
CA & 1.0 (0.029) & 0.996 (0.03) & 0.388 (0.017) & \textbf{0.386 (0.018)} & 0.669 (0.031) & 0.708 (0.032) \\
CO & 1.0 (0.061) & 0.993 (0.063) & \textbf{0.784 (0.061)} & 0.832 (0.062) & 2.105 (0.09) & 4.509 (0.216) \\
FL & 1.0 (0.06) & 1.029 (0.063) & 0.981 (0.068) & \textbf{0.939 (0.06)} & 2.311 (0.102) & 3.84 (0.169) \\
GA & 1.0 (0.044) & 0.994 (0.045) & 0.847 (0.038) & \textbf{0.673 (0.032)} & 2.293 (0.094) & 5.077 (0.22) \\
IL & 1.0 (0.077) & 0.998 (0.078) & 1.055 (0.089) & \textbf{0.828 (0.078)} & 2.18 (0.188) & 1.631 (0.138) \\
LA & 1.0 (0.056) & 0.992 (0.059) & 0.766 (0.05) & \textbf{0.294 (0.038)} & 2.921 (0.176) & 4.053 (0.204) \\
MA & 1.0 (0.036) & 0.997 (0.037) & 1.187 (0.041) & \textbf{0.417 (0.02)} & 0.64 (0.03) & 0.727 (0.045) \\
MD & 1.0 (0.059) & 1.013 (0.06) & 1.29 (0.071) & \textbf{0.696 (0.053)} & 1.78 (0.07) & 3.969 (0.191) \\
MO & 1.0 (0.078) & 1.003 (0.078) & 0.792 (0.063) & \textbf{0.379 (0.031)} & 3.59 (0.205) & 2.735 (0.127) \\
NC & 1.0 (0.037) & 1.003 (0.038) & 0.749 (0.031) & \textbf{0.417 (0.023)} & 2.722 (0.085) & 6.82 (0.171) \\
NE & 1.0 (0.051) & 1.008 (0.055) & 0.738 (0.044) & \textbf{0.135 (0.014)} & 2.871 (0.158) & 3.421 (0.149) \\
NJ & 1.0 (0.043) & 0.995 (0.043) & 1.208 (0.048) & \textbf{0.506 (0.027)} & 0.948 (0.041) & 1.054 (0.063) \\
NV & 1.0 (0.051) & 0.977 (0.051) & 0.709 (0.041) & \textbf{0.413 (0.028)} & 3.778 (0.15) & 8.172 (0.346) \\
OH & 1.0 (0.048) & 0.991 (0.052) & 0.771 (0.051) & \textbf{0.422 (0.024)} & 3.267 (0.098) & 4.19 (0.105) \\
PA & 1.0 (0.098) & 1.018 (0.103) & 0.93 (0.091) & \textbf{0.794 (0.106)} & 2.742 (0.209) & 2.433 (0.191) \\
TN & 1.0 (0.102) & 1.01 (0.104) & 0.838 (0.101) & \textbf{0.596 (0.096)} & 2.519 (0.175) & 3.751 (0.222) \\
TX & 1.0 (0.03) & 1.009 (0.031) & 0.753 (0.023) & \textbf{0.612 (0.022)} & 3.81 (0.084) & 5.975 (0.136) \\
VA & 1.0 (0.037) & 0.993 (0.038) & 1.012 (0.038) & \textbf{0.965 (0.036)} & 2.233 (0.076) & 3.466 (0.129) \\
WA & 1.0 (0.055) & 1.004 (0.062) & 1.443 (0.08) & \textbf{0.44 (0.039)} & 1.052 (0.076) & 1.057 (0.122) \\
Overall MSE & 1.0 (0.013) & 1.0 (0.013) & 0.774 (0.01) & \textbf{0.524 (0.008)} & 1.873 (0.02) & 2.895 (0.03) \\
\bottomrule
\end{tabular}
\end{sc}
\end{small}
\end{center}
\end{table}

\begin{table}[ht!]
\caption{Online News Popularity Dataset. Partition by day of the week. }
\begin{center}
\begin{small}
\begin{sc}
\begin{tabular}{lcccccc}
\toprule
Segment & \XGB & \DR & \DRSF & \MR & \MFSAN & \DADIL \\
\midrule
Monday & 1.0 (0.064) & 1.02 (0.064) & 0.987 (0.063) & \textbf{0.965 (0.063)} & 1.164 (0.064) & 1.143 (0.064) \\
Tuesday & 1.0 (0.056) & 1.011 (0.056) & 0.984 (0.056) & \textbf{0.943 (0.057)} & 1.134 (0.054) & 1.129 (0.051) \\
Wednesday & 1.0 (0.05) & 1.008 (0.05) & 0.99 (0.053) & \textbf{0.955 (0.05)} & 1.134 (0.052) & 1.175 (0.054) \\
Thursday & 1.0 (0.054) & 1.016 (0.054) & 1.013 (0.056) & \textbf{0.94 (0.052)} & 1.12 (0.047) & 1.116 (0.047) \\
Friday & 1.0 (0.053) & 0.994 (0.051) & 0.993 (0.053) & \textbf{0.984 (0.057)} & 1.189 (0.064) & 1.161 (0.06) \\
Saturday & 1.0 (0.085) & 1.009 (0.085) & 0.994 (0.085) & \textbf{0.926 (0.077)} & 1.035 (0.092) & 1.007 (0.087) \\
Sunday & 1.0 (0.078) & 1.022 (0.079) & 1.011 (0.08) & \textbf{0.873 (0.067)} & 0.974 (0.084) & 0.905 (0.066) \\
Overall MSE & 1.0 (0.023) & 1.011 (0.023) & 0.994 (0.023) & \textbf{0.949 (0.023)} & 1.129 (0.023) & 1.121 (0.023) \\
\bottomrule
\end{tabular}
\end{sc}
\end{small}
\end{center}
\end{table}


\begin{table}[ht!]
\caption{Adult Census dataset. Partition by Work Class}
\begin{center}
\begin{small}
\begin{sc}
\begin{tabular}{lcccccc}
\toprule
Segment & \XGB & \DR & \DRSF & \MR & \MFSAN & \DADIL \\
\midrule
Private & 1.0 (0.033) & 0.955 (0.033) & 0.968 (0.034) & \textbf{0.917 (0.036)} & 1.229 (0.035) & 3.582 (0.049) \\
Federal-gov & \textbf{1.0 (0.156)} & 1.003 (0.164) & 1.01 (0.133) & 1.001 (0.142) & 1.087 (0.146) & 2.181 (0.168) \\
State-gov & 1.0 (0.149) & 0.941 (0.149) & 0.921 (0.148) & \textbf{0.845 (0.139)} & 1.366 (0.158) & 3.027 (0.163) \\
Local-gov & 1.0 (0.095) & 0.93 (0.095) & 0.909 (0.098) & \textbf{0.859 (0.098)} & 1.237 (0.098) & 3.013 (0.116) \\
Self-emp-inc & 1.0 (0.098) & \textbf{0.952 (0.096)} & 1.017 (0.101) & 0.966 (0.105) & 1.158 (0.102) & 2.319 (0.179) \\
Self-emp-not-inc & 1.0 (0.074) & 0.943 (0.074) & 0.831 (0.071) & \textbf{0.796 (0.068)} & 1.253 (0.08) & 2.729 (0.091) \\
Other & 1.0 (0.135) & 0.949 (0.142) & 0.92 (0.136) & \textbf{0.909 (0.157)} & 1.289 (0.137) & 3.172 (0.13) \\
Overall CE & 1.0 (0.027) & 0.953 (0.027) & 0.948 (0.027) & \textbf{0.901 (0.028)} & 1.232 (0.028) & 3.285 (0.038) \\
\bottomrule
\end{tabular}
\end{sc}
\end{small}
\end{center}
\end{table}

\begin{table}[ht!]
\caption{Banking Promotion dataset. Partition by Occupation Category}
\begin{center}
\begin{small}
\begin{sc}
\begin{tabular}{lcccccc}
\toprule
Segment & \XGB & \DR & \DRSF & \MR & \MFSAN & \DADIL \\
\midrule
Admin. & 1.0 (0.1) & \textbf{0.981 (0.101)} & 1.02 (0.098) & 1.018 (0.102) & 1.096 (0.113) & 1.46 (0.113) \\
Blue-collar & 1.0 (0.083) & 0.969 (0.084) & 0.963 (0.081) & \textbf{0.875 (0.083)} & 1.075 (0.093) & 1.558 (0.11) \\
Entrepreneur & 1.0 (0.209) & 0.979 (0.208) & 0.979 (0.199) & \textbf{0.933 (0.2)} & 0.945 (0.175) & 1.266 (0.23) \\
Housemaid & 1.0 (0.159) & 0.955 (0.157) & 0.825 (0.147) & \textbf{0.806 (0.167)} & 1.085 (0.197) & 1.436 (0.229) \\
Management & 1.0 (0.06) & 0.977 (0.06) & 0.972 (0.058) & \textbf{0.935 (0.059)} & 0.984 (0.061) & 1.256 (0.085) \\
Retired & 1.0 (0.084) & 0.979 (0.084) & 0.923 (0.078) & \textbf{0.856 (0.074)} & 0.989 (0.091) & 1.241 (0.096) \\
Self-employed & 1.0 (0.138) & 0.978 (0.139) & 0.938 (0.141) & \textbf{0.924 (0.14)} & 0.956 (0.125) & 1.214 (0.193) \\
Services & 1.0 (0.111) & \textbf{0.967 (0.112)} & 1.028 (0.114) & 0.969 (0.126) & 1.152 (0.129) & 1.669 (0.154) \\
Student & 1.0 (0.16) & 0.98 (0.16) & 1.043 (0.157) & 0.912 (0.142) & \textbf{0.846 (0.123)} & 1.215 (0.147) \\
Technician & 1.0 (0.072) & 0.976 (0.072) & 0.957 (0.071) & \textbf{0.927 (0.075)} & 1.0 (0.079) & 1.251 (0.085) \\
Unemployed & 1.0 (0.153) & 0.992 (0.158) & 1.013 (0.168) & 1.001 (0.179) & \textbf{0.895 (0.138)} & 1.34 (0.199) \\
Unknown & 1.0 (0.297) & 0.966 (0.302) & 1.06 (0.332) & 0.987 (0.3) & \textbf{0.93 (0.281)} & 1.994 (0.417) \\
Overall CE & 1.0 (0.03) & 0.976 (0.03) & 0.971 (0.03) & \textbf{0.925 (0.03)} & 1.01 (0.032) & 1.347 (0.038) \\
\bottomrule
\end{tabular}
\end{sc}
\end{small}
\end{center}
\end{table}

\begin{table}[ht!]
\caption{Customer Segmentation Dataset. Partition by anonymized feature}
\begin{center}
\begin{small}
\begin{sc}
\begin{tabular}{lcccccc}
\toprule
Segment & \XGB & \DR & \DRSF & \MR & \MFSAN & \DADIL \\
\midrule
Cat\_1 & 1.0 (0.081) & 0.984 (0.08) & 1.001 (0.081) & 0.78 (0.061) & \textbf{0.736 (0.055)} & 1.876 (0.147) \\
Cat\_2 & 1.0 (0.041) & 0.997 (0.041) & 1.053 (0.043) & \textbf{0.771 (0.03)} & 0.793 (0.03) & 1.668 (0.069) \\
Cat\_3 & 1.0 (0.029) & 0.991 (0.03) & 1.021 (0.03) & \textbf{0.822 (0.023)} & 0.917 (0.025) & 1.594 (0.048) \\
Cat\_4 & 1.0 (0.025) & 0.975 (0.025) & 0.948 (0.024) & \textbf{0.807 (0.02)} & 0.958 (0.023) & 1.581 (0.041) \\
Cat\_5 & 1.0 (0.082) & 0.964 (0.079) & 0.997 (0.081) & 0.924 (0.074) & \textbf{0.854 (0.067)} & 1.743 (0.146) \\
Cat\_6 & 1.0 (0.012) & 0.973 (0.012) & 1.039 (0.012) & 0.846 (0.01) & \textbf{0.83 (0.009)} & 1.703 (0.021) \\
Cat\_7 & 1.0 (0.061) & 0.997 (0.062) & 1.002 (0.061) & 0.972 (0.055) & \textbf{0.932 (0.052)} & 1.673 (0.104) \\
Unknown & 1.0 (0.082) & 1.0 (0.083) & 1.032 (0.085) & \textbf{0.673 (0.053)} & 0.764 (0.059) & 1.817 (0.15) \\
Overall CE & 1.0 (0.009) & 0.978 (0.009) & 1.025 (0.01) & \textbf{0.834 (0.008)} & 0.853 (0.008) & 1.679 (0.016) \\
\bottomrule
\end{tabular}
\end{sc}
\end{small}
\end{center}
\end{table}


\begin{landscape}
\begin{table}[ht!]
\caption{User City Dataset. Relative Cross Entropy to the Production model presented for each of the country segment. We include \MRLC for our method which only includes the first stage estimators without the second stage fine-tuning.}
\label{tab: locations-test-metrics-by-country}
\begin{center}
\begin{small}
\begin{sc}
\begin{tabular}{lccccccccr}
\toprule
Country    & \XGB    & \DR & SF                     & \DRSF   & \MR     & \MRLC      & \MFSAN & \DADIL \\ \midrule
Country 1  & 0.944 (0.222)          & 0.898 (0.217)      & 0.961 (0.24)           & \textbf{0.887 (0.218)} & 1.362 (0.475)          & 1.550 (0.512) & 1.139 (0.253)         & 1.270 (0.283)         \\
Country 2  & 0.921 (0.138)          & 0.876 (0.129)      & 0.938 (0.134)          & 0.890 (0.128)          & \textbf{0.847 (0.142)} & 1.041 (0.153) & 1.008 (0.142)         & 1.635 (0.203)         \\
Country 3  & 0.957 (0.035)          & 0.911 (0.032)      & 0.855 (0.037)          & 0.835 (0.035)          & \textbf{0.763 (0.035)} & 0.842 (0.037) & 1.246 (0.038)         & 2.168 (0.056)         \\
Country 4  & 0.951 (0.128)          & 0.912 (0.122)      & 0.940 (0.131)          & 0.891 (0.121)          & \textbf{0.826 (0.108)} & 0.876 (0.111) & 1.101 (0.123)         & 1.835 (0.167)         \\
Country 5  & 0.950 (0.258)          & 0.953 (0.259)      & 0.919 (0.259)          & \textbf{0.913 (0.246)} & 0.958 (0.271)          & 1.120 (0.222) & 0.876 (0.242)         & 1.727 (0.424)         \\
Country 6  & 0.962 (0.091)          & 0.910 (0.084)      & 0.989 (0.092)          & 0.937 (0.085)          & \textbf{0.864 (0.097)} & 0.991 (0.099) & 0.995 (0.086)         & 2.135 (0.177)         \\
Country 7  & 0.923 (0.068)          & 0.920 (0.068)      & 0.956 (0.072)          & 0.908 (0.069)          & \textbf{0.832 (0.073)} & 0.943 (0.075) & 1.067 (0.074)         & 1.486 (0.104)         \\
Country 8  & 0.975 (0.085)          & 0.924 (0.085)      & 0.990 (0.085)          & 0.931 (0.083)          & \textbf{0.911 (0.105)} & 1.089 (0.107) & 1.042 (0.088)         & 2.353 (0.192)         \\
Country 9  & 0.988 (0.109)          & 0.951 (0.102)      & 0.949 (0.101)          & 0.920 (0.098)          & \textbf{0.908 (0.103)} & 1.013 (0.103) & 1.062 (0.104)         & 2.035 (0.164)         \\
Country 10 & 0.996 (0.056)          & 0.946 (0.053)      & 1.010 (0.056)          & 0.966 (0.052)          & \textbf{0.916 (0.058)} & 1.017 (0.056) & 1.234 (0.072)         & 1.524 (0.097)         \\
Country 11 & 0.995 (0.035)          & 0.954 (0.034)      & 0.995 (0.033)          & 0.966 (0.032)          & \textbf{0.914 (0.034)} & 1.042 (0.032) & 1.148 (0.040)         & 1.908 (0.072)         \\
Country 12 & 0.965 (0.072)          & 0.932 (0.071)      & 0.967 (0.072)          & 0.942 (0.07)           & \textbf{0.865 (0.075)} & 0.937 (0.067) & 1.051 (0.076)         & 1.972 (0.133)         \\
Country 13 & 0.995 (0.052)          & 0.963 (0.050)      & 0.763 (0.060)          & 0.757 (0.061)          & \textbf{0.674 (0.065)} & 0.751 (0.065) & 1.211 (0.059)         & 3.324 (0.133)         \\
Country 14 & 1.004 (0.147)          & 0.994 (0.134)      & \textbf{0.831 (0.153)} & 0.882 (0.139)          & 0.978 (0.268)          & 1.036 (0.273) & 2.016 (0.215)         & 2.243 (0.196)         \\
Country 15 & 1.051 (0.472)          & 1.099 (0.465)      & 1.177 (0.485)          & 1.021 (0.416)          & \textbf{0.86 (0.322)}  & 0.956 (0.269) & 1.705 (0.551)         & 2.938 (1.209)         \\
Country 16 & 1.048 (0.168)          & 0.963 (0.150)      & 1.016 (0.179)          & 0.868 (0.145)          & \textbf{0.807 (0.147)} & 0.920 (0.141) & 1.071 (0.145)         & 1.448 (0.191)         \\
Country 17 & \textbf{0.900 (0.192)} & 0.937 (0.207)      & 0.948 (0.206)          & 0.962 (0.204)          & 0.953 (0.237)          & 1.013 (0.251) & 1.116 (0.215)         & 2.007 (0.412)         \\
Country 18 & 1.024 (0.193)          & 0.984 (0.177)      & 0.981 (0.173)          & 0.984 (0.177)          & \textbf{0.956 (0.21)}  & 1.311 (0.284) & 1.129 (0.204)         & 1.531 (0.227)         \\
Country 19 & 0.960 (0.054)          & 0.915 (0.051)      & 0.917 (0.052)          & 0.882 (0.05)           & \textbf{0.797 (0.053)} & 0.874 (0.05)  & 1.277 (0.062)         & 1.384 (0.071)         \\
Country 20 & 0.970 (0.039)          & 0.932 (0.038)      & 0.977 (0.037)          & 0.934 (0.036)          & \textbf{0.864 (0.035)} & 0.983 (0.035) & 1.153 (0.042)         & 1.845 (0.063)         \\
Other      & 0.975 (0.020)          & 0.926 (0.019)      & 0.990 (0.019)          & 0.934 (0.018)          & \textbf{0.859 (0.019)} & 0.972 (0.018) & 1.683 (0.027)         & 1.781 (0.033)         \\ \bottomrule
\end{tabular}
\end{sc}
\end{small}
\end{center}
\end{table}
\end{landscape}

\end{document}